\documentclass{article}

\usepackage{microtype}
\usepackage{graphicx}
\usepackage{booktabs} 
\usepackage{enumitem}

\usepackage{hyperref}
\usepackage{graphicx}
\usepackage{subcaption}
\usepackage{multirow}
\usepackage{wrapfig}

\usepackage[accepted]{icml2026}

\usepackage{amsmath}
\usepackage{amssymb}
\usepackage{mathtools}
\usepackage{amsthm}
\usepackage{bbm}

\usepackage[capitalize,noabbrev]{cleveref}

\theoremstyle{plain}
\newtheorem{theorem}{Theorem}[section]

\newtheorem{lemma}[theorem]{Lemma}

\theoremstyle{definition}

\theoremstyle{remark}

\usepackage{pgfplots}
\pgfplotsset{compat=1.17}
\usepackage[textsize=tiny]{todonotes}

\icmltitlerunning{RoDiF: Robust Direct Fine-Tuning of Diffusion Policies with Corrupted Human Feedback}

\begin{document}

\twocolumn[
\icmltitle{
RoDiF: Robust Direct Fine-Tuning of Diffusion Policies \\with Corrupted Human Feedback}



\icmlsetsymbol{equal}{*}

\begin{icmlauthorlist}
\icmlauthor{Amitesh Vatsa}{asu}
\icmlauthor{Zhixian Xie}{asu}
\icmlauthor{Wanxin Jin}{asu}
\end{icmlauthorlist}

\icmlaffiliation{asu}{Intelligent Robotics and Interactive Systems (IRIS) Lab, Arizona State University, Tempe AZ, USA. 
\\
Emails: \url{amitesh.vatsa.cd.mec20@iitbhu.ac.in}, \url{zxxie@asu.edu}, \url{wanxin.jin@asu.edu}}


\icmlkeywords{Machine Learning, ICML}

\vskip 0.3in
]



\printAffiliationsAndNotice{}  

\begin{abstract}
Diffusion policies are a powerful paradigm for robotic control, but fine-tuning them with human preferences is fundamentally challenged by the multi-step structure of the denoising process. To overcome this, we introduce a Unified Markov Decision Process (MDP) formulation that coherently integrates the diffusion denoising chain with environmental dynamics, enabling reward-free Direct Preference Optimization (DPO) for diffusion policies. Building on this formulation, we propose \emph{RoDiF} (Robust Direct Fine-Tuning), a method that explicitly addresses corrupted human preferences. RoDiF reinterprets the DPO objective through a geometric hypothesis-cutting perspective and employs a conservative cutting strategy to achieve robustness without assuming any specific noise distribution. Extensive experiments on long-horizon  manipulation tasks show that RoDiF consistently outperforms state-of-the-art baselines, effectively steering pretrained   diffusion policies of diverse architectures to human-preferred modes, while maintaining strong performance even under 30\% corrupted preference labels.
\end{abstract}

\vspace{-10pt}
\section{Introduction}
\label{submission}

\begin{figure}[t]
    \centering
    \begin{subfigure}[t]{0.35\textwidth}
        \centering
        \includegraphics[width=\linewidth]{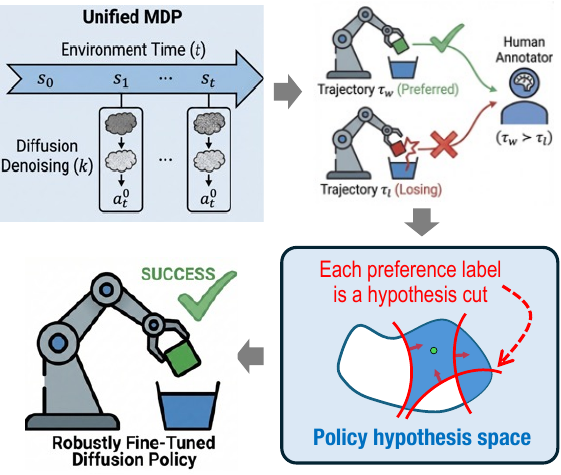}
    \end{subfigure}
    \caption{Overview of RoDiF. We introduce a Unified MDP to model the rollout of a diffusion policy, where human preferences are defined over trajectories in this unified decision process. To address noisy human feedback, RoDiF reinterprets the DPO objective through a geometric hypothesis-cutting perspective and adopts a conservative cutting strategy, enabling robust direct fine-tuning of diffusion policies.}
    \label{fig:overview}
    \vspace{-20pt}
\end{figure}

Diffusion models have recently emerged as a compelling class of policies for robotic control \cite{ho2020denoising,chi2025diffusion}, demonstrating superior capability in capturing multi-modal distributions and handling high-dimensional action spaces compared to traditional Gaussian or deterministic policies \cite{chi2025diffusion}. While Behavior Cloning (BC) with diffusion models excels at many complex robot behaviors, such as contact-rich manipulation \cite{wolf2025diffusion} and locomotion \cite{yuan2025survey}, the obtained diffusion policies are fundamentally limited by the quality of the training data. To enable robots to improve beyond their demonstrations or align with task/user specifics, it is necessary to fine-tune these policies using feedback signals.

Reinforcement Learning (RL) offers a pathway for such fine-tuning, but typically requires dense, well-engineered reward functions \cite{ma2025efficient, ren2024diffusionpolicypolicyoptimization}, which are notoriously difficult to specify for complex  tasks \cite{tan2025robo}. Consequently,  preference-based reinforcement learning (PbRL) \cite{abdelkareem2022advances} has gained traction as an alternative, allowing agents to optimize behavior based on comparative feedback (e.g., $A \succ B$) rather than scalar rewards. In the domain of Large Language Models, Direct Preference Optimization (DPO) \cite{rafailov2023direct} has established a new paradigm by deriving a closed-form objective that optimizes the policy directly against preferences, bypassing the need for an explicit reward model.

However, extending DPO-style objectives to robotic diffusion policies presents a non-trivial theoretical challenge due to a structural mismatch in the generation process. Unlike autoregressive language models, a diffusion policy operates via a multi-step denoising chain nested within the environmental MDP. Existing preference optimization methods do not natively account for this dual temporal structure—where optimization must occur over the internal denoising steps while respecting the external environmental dynamics. Furthermore, human feedback in robotic control is stochastic and often contains \emph{model-mismatched} inconsistencies (e.g., occasional label flips, annotator disagreement, or nonstationary criteria) beyond the standard Bradley--Terry noise assumption. Naively fine-tuning on such data can overfit to conflicting comparisons, effectively optimizing to satisfy erroneous constraints rather than the underlying human intent, which degrades both alignment and task performance.

In this work, we propose \textbf{RoDiF} (Robust Direct Fine-Tuning), a framework  to fine-tune pre-trained diffusion policies with noisy human preferences. We use a \textit{Unified MDP} formulation that treats the diffusion denoising steps and environmental transitions as one decision process, enabling the derivation of a tractable DPO objective for diffusion policies. To address the challenge of label noise, we reformulate the optimization objective through the geometric lens of \textit{hypothesis cutting}. By adopting a conservative cutting strategy, we derive a robust loss function that effectively ignoring the inconsistent feedback without requiring prior knowledge of  noise distribution.
Our contributions are:
\vspace{-5pt}
\begin{itemize}[leftmargin=*, nosep]
    \item We propose a \textbf{Unified MDP} formulation that integrates the diffusion denoising process with environmental dynamics, enabling direct preference optimization for diffusion policies without explicit reward modeling.
    
    \item We introduce \textbf{RoDiF}, a geometrically grounded algorithm that robustly fine-tunes diffusion policies by leveraging conservative hypothesis cutting while mitigating the impact of inconsistency preference labels.
    
    \item We empirically demonstrate that RoDiF significantly outperforms existing baselines across diverse diffusion backbones on high-dimensional, long-horizon manipulation tasks, maintaining high alignment performance even under 30\% corrupted preference labels.
\end{itemize}

\vspace{-5pt}
\section{Related Work}
\vspace{-5pt}
\subsection{Preference Based Policy Learning}
\vspace{-5pt}
\paragraph{Reward-Based Methods}
A common approach to learning policies from human preferences is to first infer a reward model from human preferences and then optimize a policy using RL \cite{christiano2017deep} or trajectory optimization \cite{jin2022learning}. Early works adopted active learning \cite{ActiveBook}, while are largely limited to linear reward parameterization.
Recent PbRL methods leverage neural networks to learn expressive reward models \cite{pebble, park2022surf, liang2022reward, liu2022meta, seneviratne2025halo},  by modeling preferences probabilistically (e.g., Bradley–Terry \cite{bradley1952rank} or Plackett–Luce \cite{plackett1975analysis}). In contrast, our method bypasses  reward learning and directly fine-tunes  diffusion policies from human preferences, leading tosimpler implementation.

\vspace{-5pt}
\paragraph{Reward-Free Methods} 
\vspace{-5pt}
To reduce computational overhead and improve efficiency, recent work \cite{an2023direct,rafailov2023direct,hejna2023inverse, hejna2024contrastive, luu2025policy, xiahuman} explores direct policy learning from human preferences without explicitly learning reward models. These methods  re-parameterize action values with policy-dependent embeddings, enabling policies to be optimized directly from preference labels. Representative approaches  include Direct Preference Optimization (DPO) \cite{rafailov2023direct} and its extensions (e.g., CPO \cite{xu2024contrastive} and SimPO \cite{meng2024simpo} to reduce computation cost, IPO \cite{azar2024general} to reduce overfitting), which  directly optimize policies under preference supervision. Our method also builds on this paradigm but through formulating sequential denoising steps and environment interactions within a unified MDP.

\vspace{-5pt}
\subsection{Robust Preference Learning}
\vspace{-5pt}
A key challenge in preference-based policy learning is the inherent noise in human preference data. To address this issue, recent robust PbRL methods employ techniques such as data filtering \cite{cheng2024rime, li2024candere, huang2025trend}, data augmentation \cite{heo2025mixing}, network architecture design \cite{xue2023reinforcement}, and hypothesis space cutting \cite{xie2025robust} to learn reliably from imperfect feedback. Robust preference-based learning has also been extended to reward-free settings: works such as \cite{ramesh2024group, kim2025lightweight, cao2025latent} adopt conservative variants of the Direct Preference Optimization (DPO) loss to mitigate overfitting to noisy labels. Our method introduces a novel geometry-inspired robust objective based on hypothesis space cutting \cite{xie2025robust}, enabling stable fine-tuning under noisy human preferences.

\vspace{-5pt}
\subsection{Diffusion Policy and its Fine-Tuning}
\vspace{-5pt}

Diffusion policies \cite{chi2025diffusion, ze20243d, xue2025reactive} achieve strong performance in robotic control but are typically trained via behavioral cloning, limiting their ability to improve beyond demonstration data. Recent work, such as DDPO \cite{Black2023TrainingDM} and D3PO \cite{Yang2023UsingHF}, formulates the diffusion denoising process as a sequential decision problem and extends Direct Preference Optimization \cite{rafailov2023direct} to enable post-learning refinement. However, these approaches focus on a static problem setting (e.g. image generation) and do not explicitly incorporate interactions with the physical environment.

When fine-tuning diffusion  policies \cite{ma2025efficient, ren2024diffusionpolicypolicyoptimization, 2025arXiv250108259C, shan2024forwardklregularizedpreference}, an additional layer of sequential decision-making arises from environment dynamics. DPPO \cite{ren2024diffusionpolicypolicyoptimization} and RSM \cite{ma2025efficient} addresses this by applying reinforcement learning with explicit rewards, while FDPP \cite{2025arXiv250108259C} extends this framework to reward-based fine-tuning from human preferences. In contrast, FKPD \cite{shan2024forwardklregularizedpreference} enables reward-free fine-tuning via KL regularization and ELBO-based approximations; however, these approximations inevitably introduce estimation bias.
Our method achieves DPO-style  fine-tuning, but jointly leveraging the full denoising trajectory and perform  direct factorization to calculate the likelihood of the action. This eliminates the need for objective approximations and reduces the bias.

\vspace{-5pt}
\section{Preliminaries}
\vspace{-5pt}
\subsection{Diffusion Policy}\label{sec:prelim-dp}
\vspace{-5pt}
Diffusion policies \cite{chi2025diffusion, ze20243d, xue2025reactive} model action chunk generation as a sequential denoising process, built on diffusion probabilistic models \cite{ho2020denoising}. Conditioned on a system state observation $\mathbf{s}$, the policy starts from a noise distribution $\mathbf{a}^K \sim \mathcal{N}(\mathbf{0}, \mathbf{I})$ and iteratively refines a noisy action sequence through a fixed $K$ number of denoising steps parameterized by $\theta$:
\begin{equation}
\begin{aligned}
\label{eq:denosing-sampler}
\small
    \mathbf{a}^{k-1} &\sim p_\theta(\mathbf{a}^{k-1} \mid \mathbf{a}^k, \mathbf{s}), \\
    p_\theta
    &\triangleq \mathcal{N}\!\left(\mathbf{a}^{k-1};\, \boldsymbol{\mu}_k\!\left(\mathbf{a}^k, \epsilon_\theta(\mathbf{a}^k, k, \mathbf{s})\right), \boldsymbol{\Sigma}_k\right),
\end{aligned}
\end{equation}
where $k = K, K-1, ..., 1$ is the step index, $\epsilon_\theta(\cdot)$ denotes a neural noise predictor, $\boldsymbol{\mu}_k(\cdot)$ is the denoising operator, and $\boldsymbol{\Sigma}_k$ is the step-dependent variance. Diffusion policies are typically learned via imitation learning by fitting $\epsilon_\theta(\cdot)$ to the injected noise in expert demonstrations. This class of policies is particularly effective at modeling complex, multi-modal action distributions \cite{wolf2025diffusion}.

\vspace{-5pt}
\subsection{Bradley-Terry Preference Model} 
\vspace{-5pt}
The Bradley--Terry model \cite{bradley1952rank} is a probabilistic formulation to model  human pairwise preferences. Given two options $a$ and $b$ with utility scores $u_a, u_b \in \mathbb{R}$, the probability that human prefers  $a$  over $b$ is defined as
\begin{equation}
\label{eq:bradley_terry}
    \mathbb{P}(a \succ b) = \frac{\exp(u_a)}{\exp(u_a) + \exp(u_b)}.
\end{equation}
In preference-based learning \cite{abdelkareem2022advances}, items typically correspond to agent trajectories or state--action pairs, and the utility score $u$ is parameterized by a reward- or policy-dependent function.

\subsection{Unified  Diffusion-Policy MDP Formulation}
\label{subsec:unified_mdp}

We consider  a diffusion policy (Section \ref{sec:prelim-dp}) acting on an environment MDP, $\mathcal{M}_{\text{env}}{=}(\mathcal{S}_{\text{env}},\mathcal{A}_{\text{env}},P_{\text{env}},R_{\text{env}},\rho_{\text{env}})$. Here,  $\mathbf{s}\in\mathcal{S}_{\text{env}}$,   $\mathbf{a}\in\mathcal{A}_{\text{env}}$, 
$\mathbf{s}_0\in\rho_{\text{env}}$, $P_{\text{env}}(\mathbf{s}_{t+1}|\mathbf{s}_t,\mathbf{a}_t)$, and $R_{\text{env}}(\mathbf{s}_t,\mathbf{a}_t)$ are the environment state, action, intial state, transition, and reward, respectively.

As a notation convention, we denote the denoising step inside the diffusion policy by the inverse-time index $k$ (superscript) and the environment rollout  step by $t$ (subscript). At the environment state $\mathbf{s}_t$  at rollout time $t$, the denoising process within the diffusion policy is modeled as a Markov Chain $(\mathbf{a}^K_t, \mathbf{a}^{K-1}_t, \dots, \mathbf{a}^{0}_t)$ over a horizon of $K+1$ steps ~\cite{black2024trainingdiffusionmodelsreinforcement}, where $\mathbf{a}_t^k$ represents the denoising state at the denoising step $k$. This denoising chain is conditioned on the environment state $\mathbf{s}_t$ and evolves following the reverse denoising process $p_\theta(\mathbf{a}^{k-1}_t|\mathbf{a}^{k}_t,\mathbf{s}_t)$ parameterized by $\theta$ in (Eq. \ref{eq:denosing-sampler}),
with drawing a random $\mathbf{a}_t^K\sim\mathcal{N}(\boldsymbol{0},\mathbf{I})$.
The environment  MDP evolves according to the trajectory  $(\mathbf{s}_0, \mathbf{a}_0, \mathbf{s}_1, \mathbf{a}_1, \dots, \mathbf{s}_t, \mathbf{a}_t, \dots)$, where each action $\mathbf{a}_t$ corresponds to the final denoising state of the corresponding diffusion  process, i.e., $\mathbf{a}_t = \mathbf{a}_t^0$.

We employ a Unified MDP formulation to integrate both the diffusion denoising process and the environment rollout process. Specifically, we define the Unified MDP as a tuple $\mathcal{M}_{\text{unified}} = (\mathcal{S}^+, \mathcal{A}^+, P^+, R^+, \rho_0^+)$. The unified system state $\mathbf{s}^+ \in \mathcal{S}^+$ and action $\mathbf{a}^+ \in \mathcal{A}^+$ are defined as:
\begin{equation}\label{eq:fmdp-state-action}
\begin{aligned}
    \mathbf{s}^+_t &\triangleq (\mathbf{a}^{K}_t, \mathbf{s}_t)
    \\
\mathbf{a}^+_t &\triangleq (\mathbf{a}^{K-1}_t, \mathbf{a}^{K-2}_t, ..., \mathbf{a}^{0}_t)
\end{aligned}
\end{equation}
The reward is defined as:
\begin{equation}
\label{eq:fmdp-reward}
    R^+(\mathbf{s}^+_t, \mathbf{a}^+_t )\triangleq R_{\text{env}}(\mathbf{s}_t,\mathbf{a}_t)=R_{\text{env}}(\mathbf{s}_t,\mathbf{a}_t^0)
\end{equation}
The transition of this flattened MDP is deifned as 
\begin{equation}
    {P}^+(\mathbf{s}^+_{t+1}|\mathbf{s}^+_t, \mathbf{a}^+_t)\triangleq 
        \mathcal{N}(0,\mathbf{I})\otimes P_{\text{env}}(\mathbf{s}_{t+1}|\mathbf{s}_{t},\mathbf{a}^{0}_t)
\end{equation}
with the initial state distribution  $\rho^+\triangleq \mathcal{N}(0,\mathbf{I})\otimes\rho_{\text{env}}$.

Under the above formulation of Unified MDP, we define its policy for the Unified MDP as $\pi (\mathbf{a}^+_t| \mathbf{s}^+_t)$. Based on the definition in (Eq. \ref{eq:fmdp-state-action}) and also the Markov properties of the denoising chain \cite{black2024trainingdiffusionmodelsreinforcement}, we have the following factorization of the Unified-MDP policy 
\begin{equation}
\label{eq:policy-chain}
     \pi_{\theta}(\mathbf{a}^+_t| \mathbf{s}^+_t) = \prod\nolimits_{k=1}^{K}p_\theta(\mathbf{a}_t^{k-1}|\mathbf{a}_t^{k},\mathbf{s}_t).
\end{equation}
where $p_\theta(\mathbf{a}^{k-1}|\mathbf{a}^{k},\mathbf{s})$ is defined in (\ref{eq:denosing-sampler}).

\subsection{DPO with Unified MDP}
\label{subsec:udpo_derivation}

Under the  Unified MDP formulation, we consider optimizing the policy $\pi$  given an optimal state-action value function $Q^*(\mathbf{s}^+,\mathbf{a}^+)$ corresponding to the \emph{unknown} reward $R^+$ in (\ref{eq:fmdp-reward}). We  define our objective as 
\begin{equation}
\label{eq:rl_objective}
        \max_{\pi} \quad \mathbb{E}_{\begin{subarray}{l} \mathbf{s}^+\sim d^{\pi} \\ \mathbf{a}^+\sim\pi(\cdot|\mathbf{s}^+) \end{subarray}}
        Q^*(\mathbf{s}^+,\mathbf{a}^+)-\beta \mathbf{KL}(\pi||\pi_{\text{ref}})
\end{equation}
where $\pi_{\text{ref}}(\mathbf{a}^+|\mathbf{s}^+)$ is a reference policy for the unified MDP. The above objective means that the policy needs to be fine-tuned to maximize the  value of $Q^*$, while a regularization term $-\beta \mathbf{KL}(\pi||\pi_{\text{ref}})$ prevents drifting away from the reference policy $\pi_{\text{ref}}$. $\beta>0$ controls the regulization strength.

Following \cite{peters2007reinforcement, peng2019advantage},  \eqref{eq:rl_objective} permits a closed-form optimal policy  given an $Q^*$ function.
\begin{equation}
\small
    \pi^*(\mathbf{a}^+|\mathbf{s}^+) {=} \frac{1}{Z(\mathbf{s}^+)} \pi_{\text{ref}}(\mathbf{a}^+|\mathbf{s}^+) \exp \left( \frac{1}{\beta} Q^*(\mathbf{s}^+, \mathbf{a}^+) \right),
    \label{eq:optimal_policy}
\end{equation}
where $Q^*$ is the given optimal state-action value function and $Z(s^+)$ is the partition function. 

Rearranging Eq.~\eqref{eq:optimal_policy}, we can express the optimal $Q^*$ purely in terms of the likelihood ratio between the optimal policy and the reference policy:
\begin{equation}
    Q^*(\mathbf{s}^+, \mathbf{a}^+) = \beta \log \frac{\pi^*(\mathbf{a}^+|\mathbf{s}^+)}{\pi_{\text{ref}}(\mathbf{a}^+|\mathbf{s}^+)} + \beta \log Z(\mathbf{s}^+).
    \label{eq:q_value_inversion}
\end{equation}
The above equation is a critical reparameterization trick \cite{rafailov2023direct} in the following preference optimization: the policy itself can be used to express a corresponding state-action value, eliminating the need to approximating or compute $Q^*$ itself.

\vspace{-5pt}
\subsubsection{Advantage Bradley-Terry Preference}
\vspace{-5pt}
Typically, Bradley-Terry preference model \eqref{eq:bradley_terry} is defined  on the reward function ${R}^+$. However,  such models have been shown to be inconsistent with real human preferences \cite{knox2024learning}: for instance, consider a sparse goal-reaching reward . Two segments that do not reach the goal would have the same returns even if one moved towards the goal while the other moved away from it. Thus, we follow \cite{knox2024learning,hejna2023contrastive} and choose per-step advantage function for the human preference model. 

Specifically, define the human's implicit state-action value function $Q^*(\mathbf{s}^+,\mathbf{a}^+)$ and  value function $V^*(\mathbf{s}^+)$. Human's implicit advantage function is defined as $A^*(\mathbf{s}^+,\mathbf{a}^+)\triangleq Q^*(\mathbf{s}^+,\mathbf{a}^+){-}V^*(\mathbf{s}^+)$.
Consider two actions  $(\mathbf{a}^{+}_w,\mathbf{a}^{+}_l)$ drawn from   policy   $\pi(\cdot | \mathbf{s}^+ )$, The human ranks the two actions as $(\mathbf{a}^{+}_w|\mathbf{s}^+\succ\mathbf{a}^{+}_{l}|\mathbf{s}^+)$, i.e., $\mathbf{a}^{+}_w$ is preferred than $\mathbf{a}^{+}_{l}$,  according to the probability
\begin{equation}
        P(\mathbf{a}^{+}_w|\mathbf{s}^+\succ\mathbf{a}^{+}_{l}|\mathbf{s}^+) {=}\frac{e^{ A^*(\mathbf{a}^{+}_w,\mathbf{s}^+)/\alpha}}{e^{ A^*(\mathbf{a}^{+}_w,\mathbf{s}^+)/\alpha}{+}e^{A^*(\mathbf{a}^{+}_l,\mathbf{s}^+)/\alpha}}\\
    \label{eq:bradley_terry_advantage}
\end{equation}
where $\alpha>0$ is the temperature parameter for the  Boltzmann distribution (intuitively, the human rationality). By simplification, the above \eqref{eq:bradley_terry_advantage} can be reduced to  
\begin{equation}
\label{eq:BT-model}
\begin{aligned}    P(\mathbf{a}^{+}_w|\mathbf{s}^+\succ\mathbf{a}^{+}_{l}|\mathbf{s}^+)&=\sigma_{\alpha} \left(  \Delta Q^* (\mathbf{a}^{+}_w|\mathbf{s}^+, \mathbf{a}^{+}_l|\mathbf{s}^+)
     \right)\\
    \Delta Q^* (\mathbf{a}^{+}_w|\mathbf{s}^+, \mathbf{a}^{+}_l|\mathbf{s}^+)&\triangleq  Q^*(\mathbf{s}^+, \mathbf{a}^+_w) 
     -  Q^*(\mathbf{s}^+, \mathbf{a}^+_l).
\end{aligned}
\end{equation}
We use the action-value function $Q^*$ rather than the advantage function $A^*$ in \eqref{eq:BT-model}, since the state-value term $V^*(\mathbf{s}^+)$ cancels out when taking differences between advantages.

\vspace{-5pt}
\subsubsection{Per-Control DPO for Diffusion Policies}
\vspace{-5pt}
Under the reparameterization trick in (\ref{eq:q_value_inversion}), we can parameterize the unknown human $Q^*$ function in \eqref{eq:BT-model} with the parametrized policy $\pi_{\theta}$. Also, considering the definition and Markov property of the diffusion policy $\pi_\theta$ in (\ref{eq:policy-chain}) in the Unified MDP, we have 
\vspace{-5pt}
\begin{multline}
\label{eq:c-function-diffusion}
         \Delta Q_{\theta} (\mathbf{a}^{+}_{w}|\mathbf{s}^+,\mathbf{a}^{+}_{l}|\mathbf{s}^+)=\beta {\sum_{k=1}^K \log \frac{p_\theta(\mathbf{a}_{ w}^{k-1} | \mathbf{a}_{ w}^{k}, \mathbf{s})}{\pi_{\mathrm{ref}}(\mathbf{a}_{ w}^{k-1} | \mathbf{a}_{w}^{k}, \mathbf{s})}} \\
     - \beta {\sum\nolimits_{k=1}^K \log \frac{p_\theta(\mathbf{a}_{ l}^{k-1} | \mathbf{a}_{ l}^{k}, \mathbf{s})}{\pi_{\mathrm{ref}}(\mathbf{a}_{ l}^{k-1} | \mathbf{a}_{ l}^{k}, \mathbf{s})}},
\end{multline}
Thus, given a dataset of human preference labels for per-control of diffusion policy  $\mathcal{D}=\{(\mathbf{a}^{+}_{i,w}|\mathbf{s}^+_i\succ\mathbf{a}^{+}_{i,l}|\mathbf{s}^+_i)\}_{i=1}^{N}$ the  DPO objective for fine-tuning the diffusion policy policy $\pi_{\theta}$ can be formulated as maximizing the likelihood \cite{rafailov2023direct} of the advantage preference probabilities (\ref{eq:BT-model}):
\begin{equation}
    \label{eq:dpo-peraction}
        \mathcal{L}_{\text{DPO}}(\theta){=}-\sum\nolimits_{i=1}^N
    \log\sigma_{\alpha}
    \left(\Delta Q_{\theta} (\mathbf{a}^{+}_{i,w}|\mathbf{s}_i^+,\mathbf{a}^{+}_{i,l}|\mathbf{s}_i^+)
    \right).
\end{equation}

\vspace{-5pt}
\subsubsection{Extension to Trajectory Comparison}
\label{sec:traj_compare}
\vspace{-5pt}
While the above DPO formulation of per-control preference has nice theoretical meaning, but it can be incontinent for practice because providing per-control preference labels $\mathcal{D}$ can be labor-consuming for human users. Instead, it's more viable for human users to provide per-trajectory/episode preference labels. Suppose a diffusion policy $\pi_{\theta}$ generates two trajectories that are compared by a human observer, denoted as $(\tau_w\succ\tau_r)$, where $\tau_w$ and $\tau_r$ represent the winning (preferred) and losing (not preferred) trajectories, respectively.
We can leverage the assumption: if a trajectory $\tau_w$ is preferred over $\tau_l$, then every constituent state-action pair $(\mathbf{s}_{t,w}^+, \mathbf{a}_{t,w}^+)$ in the winning trajectory $\tau_w$ is likely superior to the corresponding pair $(\mathbf{s}_{t,l}^+, \mathbf{a}_{t,l}^+)$  in  $\tau_l$. Accordingly, for  preference label of a pair $(\tau_w, \tau_l)$ of trajectories of length $T$, we perform $T$ comparisons of $\Delta Q_\theta(\mathbf{a}^{+}_{t,w}|\mathbf{s}_{t,w}^+,\mathbf{a}^{+}_{t,l}|\mathbf{s}_{t,l}^+)$ and compute the total loss
\begin{equation}
\label{eq:segment-dpo}
            \mathcal{L}(\theta, \tau_w\succ \tau_l) \triangleq{-} \sum_{t=0}^{T-1} \log \sigma_{\alpha} \Big( \Delta Q_\theta(\mathbf{a}^{+}_{t,w}|\mathbf{s}_{t,w}^+,\mathbf{a}^{+}_{t,l}|\mathbf{s}_{t,l}^+) \Big)
\end{equation}
Given $M$ trajectory pairs,  the per-control diffusion policy DPO loss (\ref{eq:dpo-peraction}) can be extended into 
\begin{equation}
\begin{aligned}
\mathcal{L}_{\text{DP-DPO}}(\theta) &\triangleq\sum\nolimits_{i=1}^{M} \mathcal{L}(\theta, \tau_w^{(i)}\succ \tau_l^{(i)}) 
\end{aligned}
\label{eq:dense_loss}
\end{equation}

\vspace{-5pt}
\section{Robust Direct Fine-Tuning}
\label{sec:robust_finetuning}
\vspace{-0pt}
The standard DPO formulation assumes preferences derive from a single Bradley-Terry model. While this captures intrinsic ambiguity, it fails to account for model-mismatched inconsistencies, such as label flips or annotator errors, which we term \textbf{corrupted preference labels}. Since these out-of-distribution labels render optimization brittle \cite{lee2021b}, we propose a fine-tuning objective robust to such corruption without assuming a known corruption data distribution.

\vspace{-5pt}
\subsection{Interpreting DPO Objective as Hypothesis Cutting}
\label{subsec:geometric_dense}
\vspace{-5pt}
Let $\Theta$ be the entire hypothesis space for the parameterized diffusion policies (\ref{eq:policy-chain}). 
Define $\theta^H \in \Theta$  the true diffusion policy parameter aligned with human preferences. Based on the reparameterization in~\eqref{eq:c-function-diffusion}, $\Delta Q_{\theta^H} = \Delta Q^*$ in (\ref{eq:BT-model}).

In \eqref{eq:BT-model}, if we let the temperature $\alpha\rightarrow 0$,  $\sigma_\alpha(\cdot)$ will become the  Heaviside step function  $\mathbf{H}(\cdot)$, and Bradley-Terry model becomes
\begin{equation}
\label{eq:consistent_label}
P(\mathbf{a}^{+}_w|\mathbf{s}^+\succ\mathbf{a}^{+}_{l}|\mathbf{s}^+){=}\begin{cases}
        1&  \Delta Q^* (\mathbf{a}^{+}_{w}|\mathbf{s}^+,\mathbf{a}^{+}_{l}|\mathbf{s}^+)\geq0 \\
        0& \text{otherwise}
    \end{cases}.
\end{equation}
Due to $\Delta Q_{\theta^H} = \Delta Q^*$, we can interpret from (\ref{eq:consistent_label}) that each human preference pair $(\mathbf{a}^{+}_{i,w}|\mathbf{s}_i^+ \succ\mathbf{a}^{+}_{i,l}|\mathbf{s}_i^+)$ will induce an inequality constraint $\mathcal{C}_i$  on the policy parameter $\theta$:
\begin{equation}
\label{eq:cut_i}
    \mathcal{C}_i\triangleq\left\{\theta\in\Theta\,|\,\Delta Q_{\theta} \big(\mathbf{a}^{+}_{i,w}|\mathbf{s}_i^+,\mathbf{a}^{+}_{i,l}|\mathbf{s}_i^+\big)\geq0 \right\}.
\end{equation}
and the true policy parameter $\theta^{H}\in \mathcal{C}_i$.

In another perspective, as  in Fig. \ref{fig:cuts} (left), $\mathcal{C}_i$ can be interpreted as a \textbf{cut} of the hypothesis space $\Theta$, removing the hypothesis space portion which does not contain  $\theta^H$. In this geometry interpretation, $\sigma_{\alpha\rightarrow 0}(\Delta Q_{\theta})=\mathbf{H}(\Delta Q_{\theta})$ becomes an indicator function of the cut $\mathcal{C}_i$, 
\begin{equation}\label{eq:voting}
    \mathbf{H}\left(\Delta Q_{\theta} (\mathbf{a}^{+}_{i,w}|\mathbf{s}_i^+,\mathbf{a}^{+}_{i,l}|\mathbf{s}_i^+)
    \right){=}\begin{cases}
        1& \theta\in\mathcal{C}_i \\
        0& \text{else}
    \end{cases}
\end{equation}
\vspace{-5pt}
Therefore, the diffusion policy 
DPO loss  (\ref{eq:dpo-peraction}) equivalent to 
\begin{equation}
\label{eq:prod-voting}
   \max_\theta \quad \prod\nolimits_{i=1}^N \mathbf{H}\left(\Delta Q_{\theta} (\mathbf{a}^{+}_{i,w}|\mathbf{s}_i^+,\mathbf{a}^{+}_{i,l}|\mathbf{s}_i^+)
    \right),
\end{equation}
where the product of indicator functions  is equivalent to taking the intersection of all preference-induced cuts $\mathcal{C}_i$, $i=1,2,..,N$.
As more cuts are applied, the interacting hypothesis space gets progressively \emph{shrunk}, eventually, we can retain $\theta^H$ from the shrinking interaction set $\theta^H \in \Theta \cap \mathcal{C}_1\cap\mathcal{C}_2...\cap\mathcal{C}_N$.

\vspace{-5pt}
\subsection{Robust Direct Fine-Tuning Objective}
\vspace{-5pt}

\begin{figure}[h]
    \centering
    \begin{subfigure}[t]{0.18\textwidth}
        \centering
        \includegraphics[width=\linewidth]{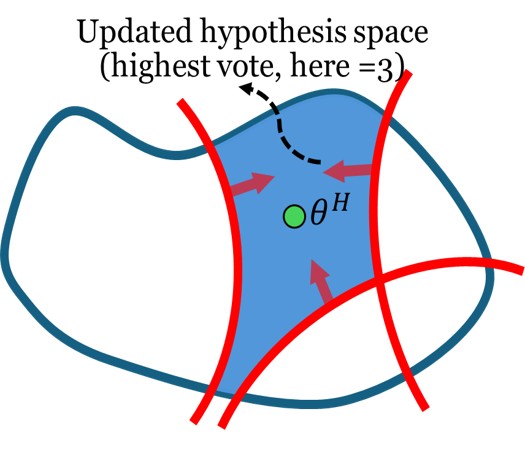}
    \end{subfigure}
    \hspace{20pt}
    \begin{subfigure}[t]{0.18\textwidth}
        \centering
        \includegraphics[width=\linewidth]{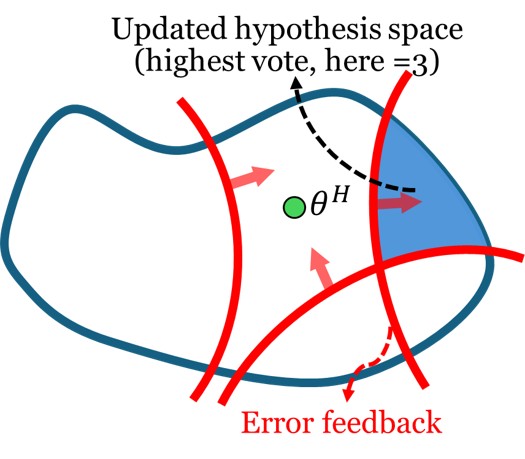}
    \end{subfigure}
    \vspace{-5pt}
    \caption{Illustration of intersecting all induced cuts. Left: all three preferences are consistent with \eqref{eq:consistent_label}, The aligned parameter $\theta^H$ falls into the intersection. Right: In case of one corrupted preference, the updated hypothesis space does not contain $\theta^H$.}
    \label{fig:cuts}
    \vspace{-10pt}
\end{figure}

Human preference labels can be inherently inconsistent, e.g., $(\mathbf{a}^+_w|\mathbf{s}^+ \succ \mathbf{a}^+_l|\mathbf{s}^+)$ is labeled as $(\mathbf{a}^+_w|\mathbf{s}^+ \prec \mathbf{a}^+_l|\mathbf{s}^+)$), due to factors such as fatigue/errors, disagreement across annotators, or nonstationary criteria over time. 
Thus, in the following we will develop a robust  mechanism based on the above geometric interpretation of hypothesis cutting. 

To proceed, we view each $\mathbf{H}(\Delta Q_{\theta}(\cdot,\cdot))$ in \eqref{eq:voting} as a voting function over the hypothesis space: Among the hypotheses in $\Theta$,  $\theta$s that satisfy the cut $\mathcal{C}_i$ in (\ref{eq:cut_i}) receive ``+1" vote, since $\mathbf{H}(\Delta Q_{\theta}(\cdot,\cdot))$ outputs 1. In contrast, parameters that violate the cut receive ``0" vote. We use $V_{\mathcal{C}_i}(\theta)$ to denote the vote generated by $i$-th cut $\mathcal{C}_i$ (\ref{eq:cut_i}):
\begin{equation}
    \label{eq:single_vote}
    V_{\mathcal{C}_i}(\theta)\triangleq \mathbf{H}\left(\Delta Q_{\theta} (\mathbf{a}^{+}_{i,w}|\mathbf{s}_i^+,\mathbf{a}^{+}_{i,l}|\mathbf{s}_i^+)\right)
\end{equation}

With this notation of the votes, the following lemma  states an equivalent solution set  to (\ref{eq:prod-voting}) and also the implication if the human preference label set contains corrupted labels.
\begin{lemma}
\label{lemma:cuts}
 Consider a  dataset of human preference labels  $\mathcal{D}{=}\{(\mathbf{a}^{+}_{i,w}|\mathbf{s}^+_i\succ\mathbf{a}^{+}_{i,l}|\mathbf{s}^+_i)\}_{i=1}^{N}$, and define $\theta^H$ as true diffusion policy parameter aligned with human preference. If all preference labels in $\mathcal{D}$ is  consistent with (\ref{eq:consistent_label}), then
\begin{equation}
    \theta^H\in\arg \max_\theta \quad \sum\nolimits_{i=1}^{N}V_{\mathcal{C}_i}(\theta)
    \label{eq:perfect_indicator}
\end{equation}
Otherwise when $\mathcal{D}$ contain corrupted preference labels, i.e.,  $\exists j, \Delta Q^*\big(\mathbf{a}^{+}_{j,w}|\mathbf{s}_j^+,\mathbf{a}^{+}_{j,l}|\mathbf{s}_j^+\big) < 0$, then $\theta^H$ cannot be found in the above solution set. 
\end{lemma}
\vspace{-15pt}
\begin{proof}
   The proof is given in Appendix \ref{appendix:proof1}.
\end{proof}
\vspace{-15pt}
We explain the above lemma using Fig.~\ref{fig:cuts} (right). Any corrupted preference label $(\mathbf{a}^{+}_{j,w}|\mathbf{s}_j^+,\mathbf{a}^{+}_{j,l}|\mathbf{s}_j^+)$ means that $\theta^H\notin\mathcal{C}_j$, i.e., the induced cut will  exclude the true $\theta^H$ from the current hypothesis space. Once $\theta^H$ is removed, all subsequent updates operate on a restricted hypothesis  space that no longer contains $\theta^H$, demonstrating how a single incorrect preference can irreversibly derail the learning process. 

To still robustly find the true policy $\theta^H$ under human corrupted preference label, we  employ a conservative voting strategy. The intuitive idea is stated below.  Reconsider the example in Fig.~\ref{fig:cuts} (right), where three hypothesis cuts are made, one of which corresponds to corrupted preference label. If we pick the regions of greater or equal two votes,  shown in Fig.~\ref{fig:robust}, then we can still guarantee the remaining hypothesis space (blue region in Fig.~\ref{fig:robust}) to contain $\theta^H$. 
\begin{figure}
    \centering
    \includegraphics[width=0.45\linewidth]{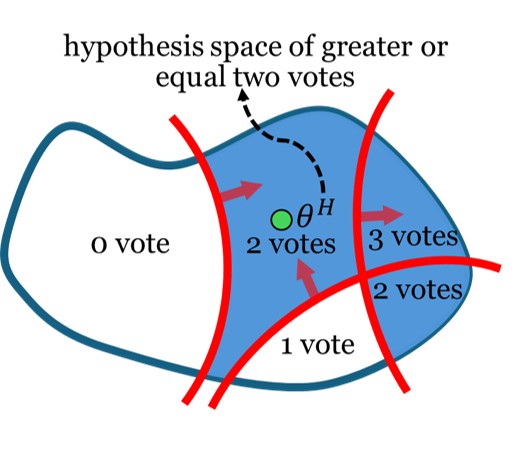}
    \vspace{-5pt}
    \caption{A conservative voting strategy to ensure robustness. By retaining only the hypotheses that receive at least two votes, the true policy parameter $\theta^H$ remains in the updated hypothesis space even when one out of three preferences is corrupted.}
    \label{fig:robust}
    \vspace{-15pt}
\end{figure}

Formally, we introduce a \textbf{conservativeness level} $\gamma \in [0,1]$, which represents an estimated upper bound on the ratio of the corrupted labels, meaning that \emph{at least} $\lfloor(1-\gamma)N\rfloor$ labels of given $N$ preference labels satisfy (\ref{eq:consistent_label}). 
Then, we have the following lemma stating a robust loss, whose solution  contains the true aligned diffusion policy parameter $\theta^H$.
\begin{lemma}
\label{lemma:robust}
Given a  dataset of $N$ human preference labels and   a \emph{conservative level} $\gamma\in(0,1)$ for conservative estimate of actual label corruption rate, i.e. \emph{at least} $\lfloor(1-\gamma)N\rfloor$ labels  satisfy (\ref{eq:consistent_label}), the human-aligned diffusion policy parameter is the solution to the  \textbf{robust direct fine-tuning loss} below
\vspace{-5pt}
    \begin{equation}
    \label{eq:robust_dpo_loss}
    \theta^H\in \arg\max_{\theta}\mathbf{H}\left(\sum_{i=1}^{N}V_{\mathcal{C}_i}(\theta) {-} \lfloor(1{-}\gamma)N\rfloor {+} 0.5 \right).
\end{equation}
\end{lemma}
\vspace{-15pt}
\begin{proof}
   The proof is given in Appendix \ref{appendix:proof2}.
\end{proof}
\vspace{-5pt}

Geometrically, the above direct fine-tuning loss relaxes the strict requirement of  all consistent cuts: a parameter $\theta$ remains in the hypothesis space even if it violates a limited number of cuts. In particular, the policy is allowed to ignore up to $\gamma N$ corruption cuts, thereby mitigating catastrophic “false cuts” caused by a small fraction of corrupted  labels. Specifically, if $\gamma$ is chosen as zero, the solution set in \eqref{eq:robust_dpo_loss} becomes the same as the one in \eqref{eq:perfect_indicator}, i.e.,  the corruption-free fine-tuning is a special case of robust fine-tuning.

\vspace{-5pt}
\subsection{Practical Loss for Robust Direct Fine-Tuning}
\label{subsec:robust_loss}
\vspace{-5pt}
To enable gradient-based optimization with the robust direct fine-tuning loss in (\ref{eq:robust_dpo_loss}), we soften the non-differentiable Heaviside step function $\mathbf{H}$ in \eqref{eq:single_vote} and \eqref{eq:robust_dpo_loss} using a sigmoid function $\sigma_\alpha(x) = \frac{1}{1 + e^{-x/\alpha }}$ with temperature $\alpha>0$.

For the trajectory pairs (recall Section \ref{sec:traj_compare}), we have the dataset $\{(\tau_w^{(i)}, \tau_l^{(i)})\}_{i=1}^{M}$ with each  trajectory of length $T$. Define the per-pair, per-step preference voting function (\ref{eq:single_vote}) by softening $\mathbf{H}$ using sigmoid function $\sigma_\alpha$
\begin{equation}
    \hat{V}_{\mathcal{C}_{i,t}}\triangleq\sigma_{\alpha} \Big( \Delta Q_\theta(\mathbf{a}^{+}_{i,t,w}|\mathbf{s}_{i,t,w}^+,\mathbf{a}^{+}_{i,t,l}|\mathbf{s}_{i,t,l}^+) \Big),
\end{equation}
for $i=1,2,...M$, $t=1,2,...,T$, with $ \Delta Q_\theta$ defined using the reparameterization   (\ref{eq:c-function-diffusion}), 
A soften version of the robust direct fine-tuning loss  (\ref{eq:robust_dpo_loss}) is defined 
\begin{equation}
    \mathcal{L}_{\text{RoDiF}}(\theta) = - \log \sigma_{\alpha} \left( \sum_{i=1}^{M}\sum_{t=1}^{T}\hat{V}_{\mathcal{C}_{i,t}}(\theta) - \nu(1-\gamma)MT \right).
    \label{eq:final_robust_loss}
\end{equation}
where $\nu \approx 1$ is a scaling factor which replaces the floor operation in  (\ref{eq:robust_dpo_loss}).


\section{Experiments}
\label{sec:experiments}

\vspace{-5pt}
In experiment,  we aim to answer the following questions:
\textbf{(1) Alignment Effectiveness:} Can RoDif effectively steer a pretrained multi-modal policy toward a  preferred behavior without degrading the policy's stability?
\textbf{(2) Robustness to Noise:} Does RoDif maintain performance when the preference dataset contains different levels of corrupted labels? \textbf{(3) Architecture Agnosticism:} Is our  method effective across different diffusion  backbones?

\begin{figure}[h]
    \centering
    \begin{subfigure}[t]{0.11\textwidth}
        \centering
        \includegraphics[width=\linewidth]{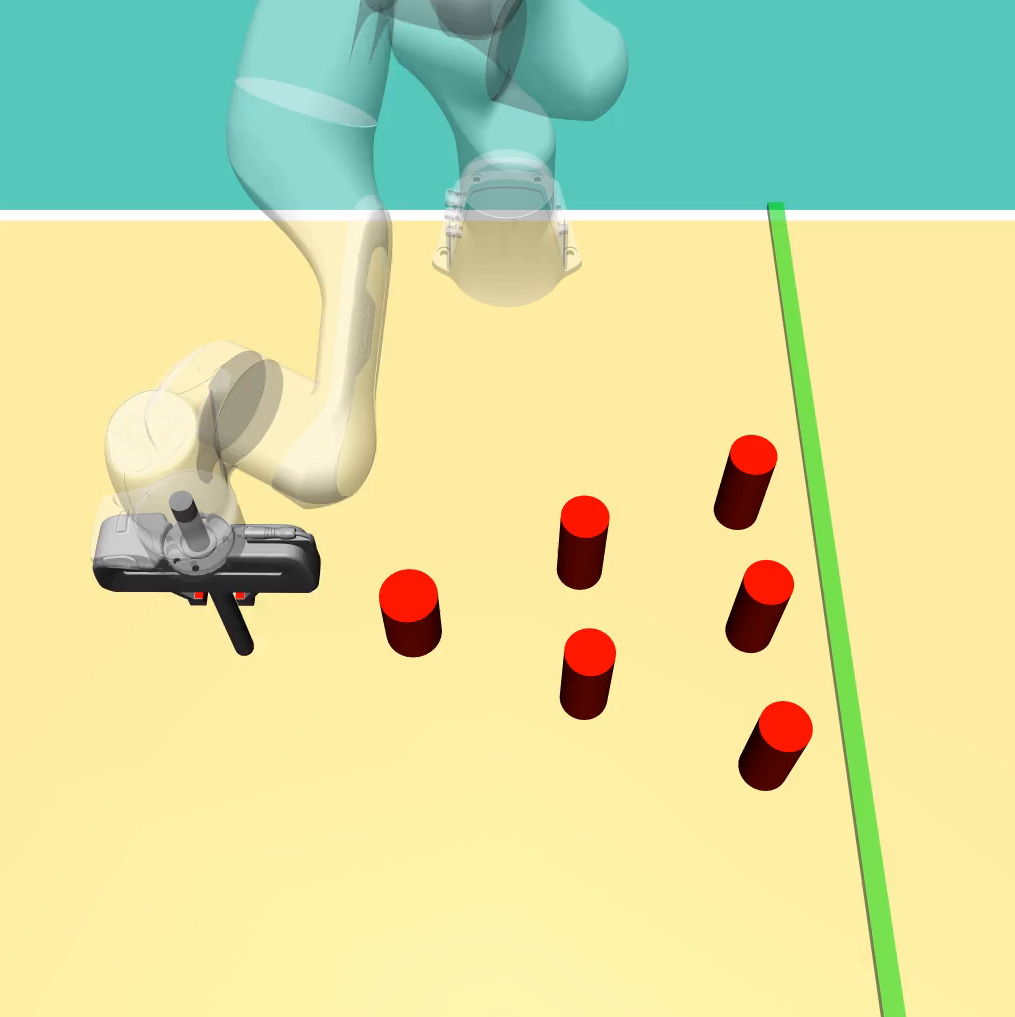}
        \caption{Avoid}
        \label{fig:task_avoid}
    \end{subfigure}\hspace{10pt}
    \begin{subfigure}[t]{0.11\textwidth}
        \centering
        \includegraphics[width=\linewidth]{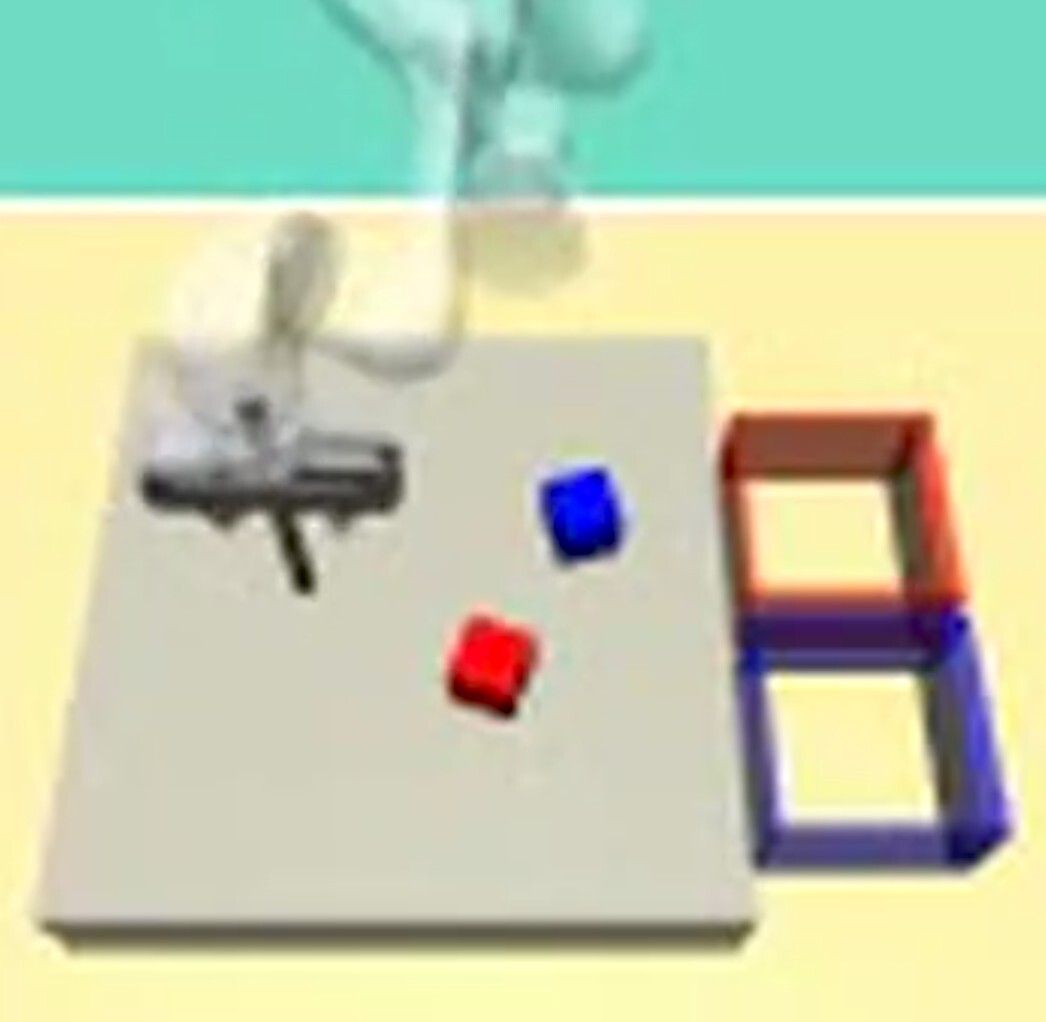}
        \caption{Sort}
        \label{fig:task_sort}
    \end{subfigure}\hspace{10pt}
    \begin{subfigure}[t]{0.11\textwidth}
        \centering
        \includegraphics[width=\linewidth]{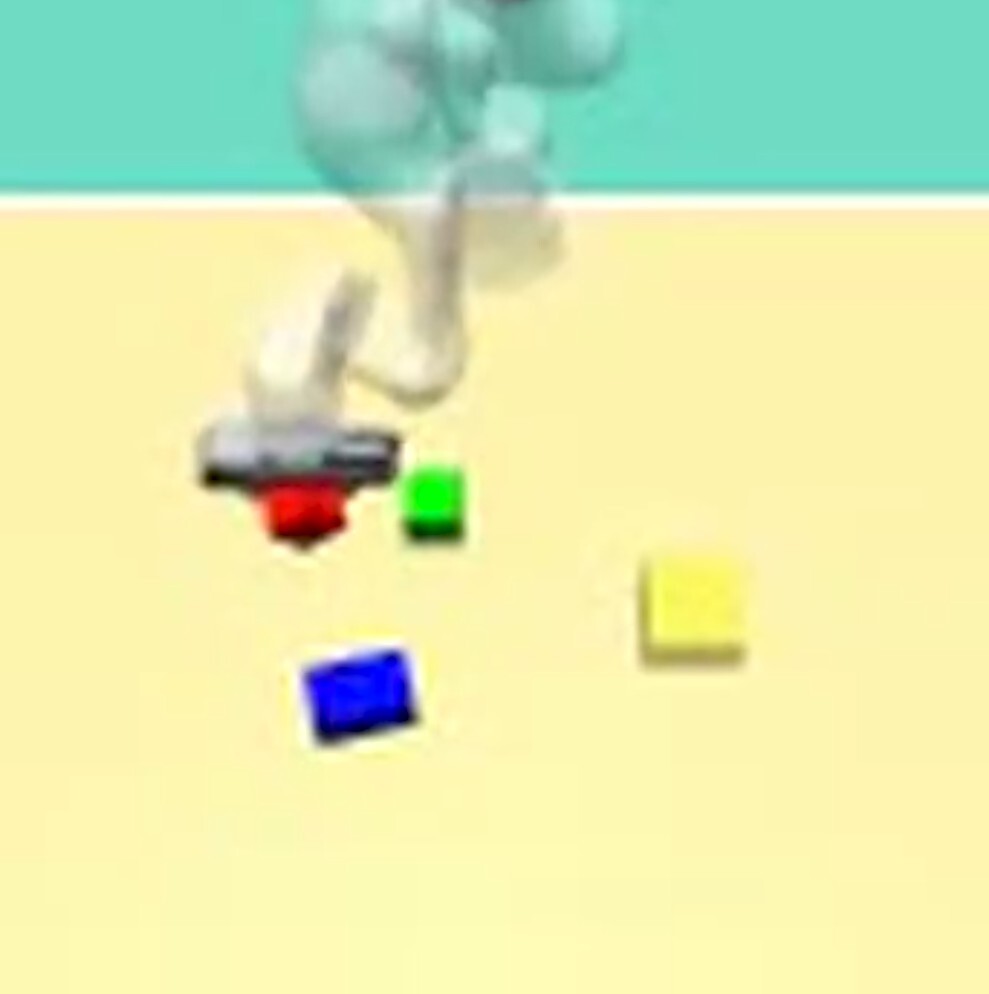}
        \caption{Stack}
        \label{fig:task_stack}
    \end{subfigure}\\
    \begin{subfigure}[t]{0.11\textwidth}
        \centering
        \includegraphics[width=\linewidth]{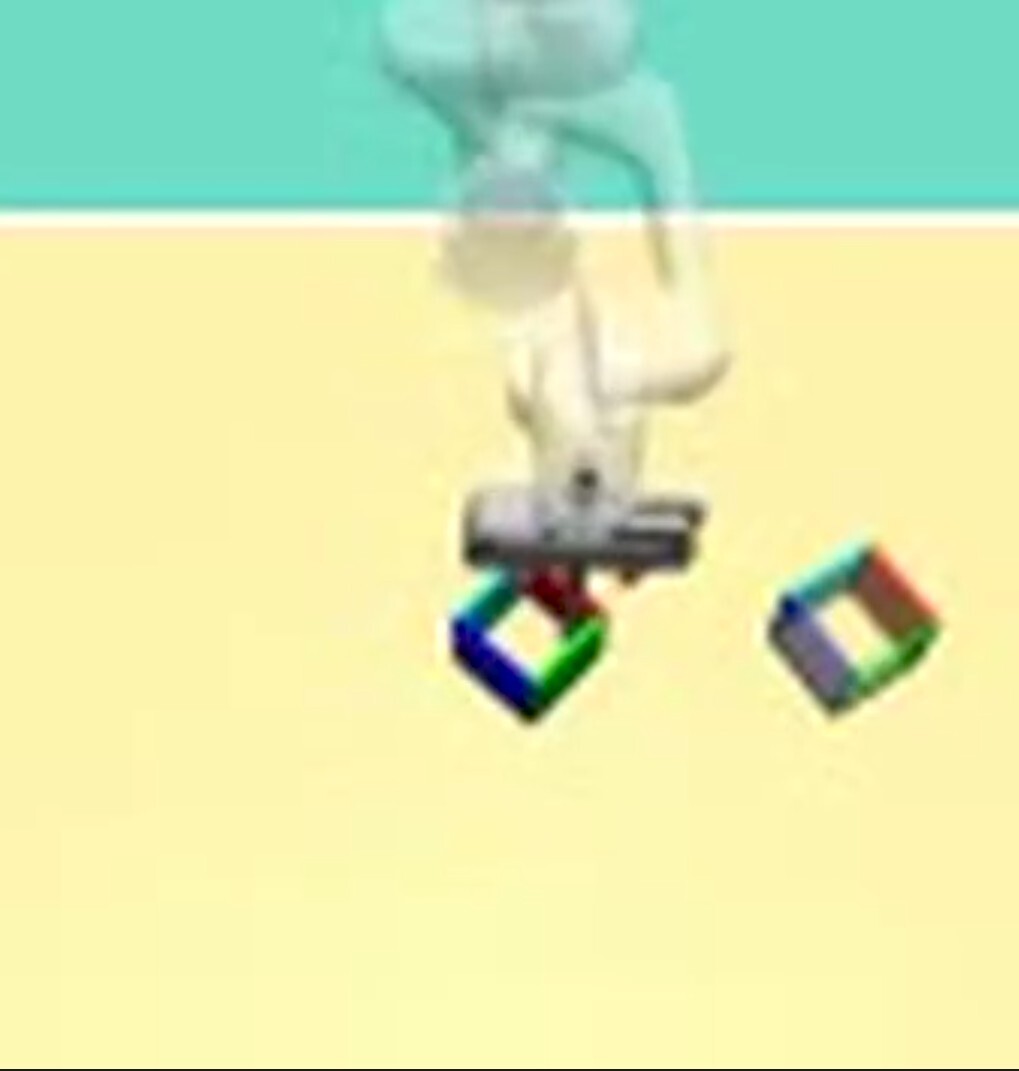}
        \caption{Align}
        \label{fig:task_align}
    \end{subfigure}
    \hspace{10pt}
        \begin{subfigure}[t]{0.11\textwidth}
        \centering
        \includegraphics[width=\linewidth]{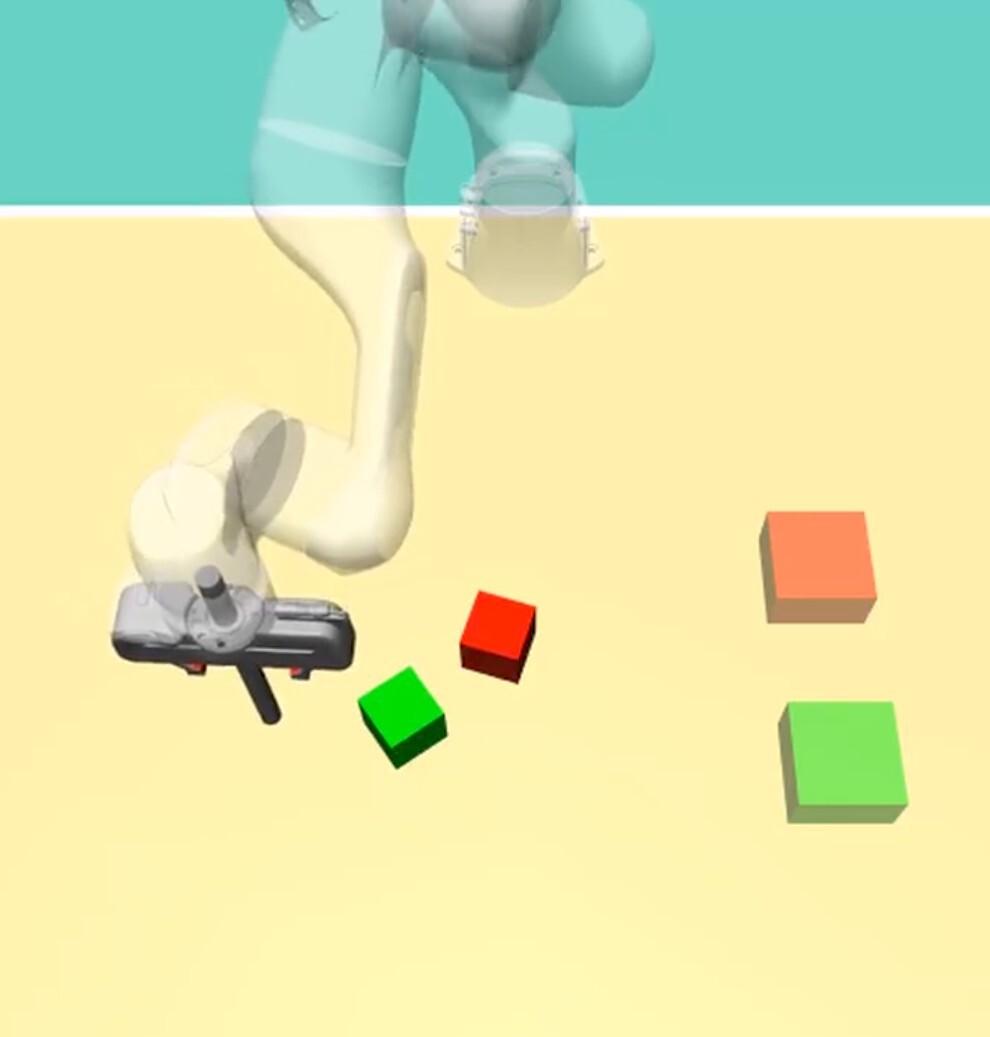}
        \caption{Pushing}
        \label{fig:task_push}
    \end{subfigure}
    \caption{Five long-horizon manipulation tasks adapted from D3IL benchmark. \cite{jia2024towards}}
    \label{fig:task-setup}
    \vspace{-10pt}
\end{figure}


\begin{table*}[!htbp]
    \centering
    \caption{The setup of five  D3IL  Benchmark tasks \cite{jia2024towards} and our preference definition for each task.}
    \label{tab:task_details}
    \resizebox{1.0\textwidth}{!}{%
    \begin{tabular}{l l l l l}
        \toprule
        \textbf{Task Name} & \textbf{Description (how to define success)} & \textbf{Observation space} & \textbf{Pretrained Behavior (Modes)} & \textbf{Mode Preference (Winner)} \\
        \midrule
        \textbf{Avoid} & reach finish line while avoiding obstacles & 4-dim (EE pos, goal) & Pass left or right & \textbf{Left-side} trajectory \\
        \textbf{Pushing} & Push red and green blocks to  matching targets & 10-dim (EE pose, block pos/orient) & Diverse sorting orders & \textbf{Red-to-Red} \\
        \textbf{Sorting} & Sort red and blue blocks into  matching boxes & $(4+3X)$-dim ($X=2$ blocks) & Arbitrary sorting sequences & \textbf{Red-then-Blue} sequence \\
        \textbf{Align} & Push hollow box to target pose & $2 \times 96 \times 96$ RGB + Proprioception & Push from inside or outside & Push from \textbf{Inside} \\
        \textbf{Stacking} & Stack three colored blocks into a tower & $2 \times 96 \times 96$ RGB + Proprioception & 6 stacking permutations & \textbf{Red $\rightarrow$ Blue $\rightarrow$ Green} \\
        \bottomrule
    \end{tabular}%
    }
\end{table*}

\begin{figure*}[!htbp]
\centering
\centerline{\includegraphics[width=0.85\textwidth]{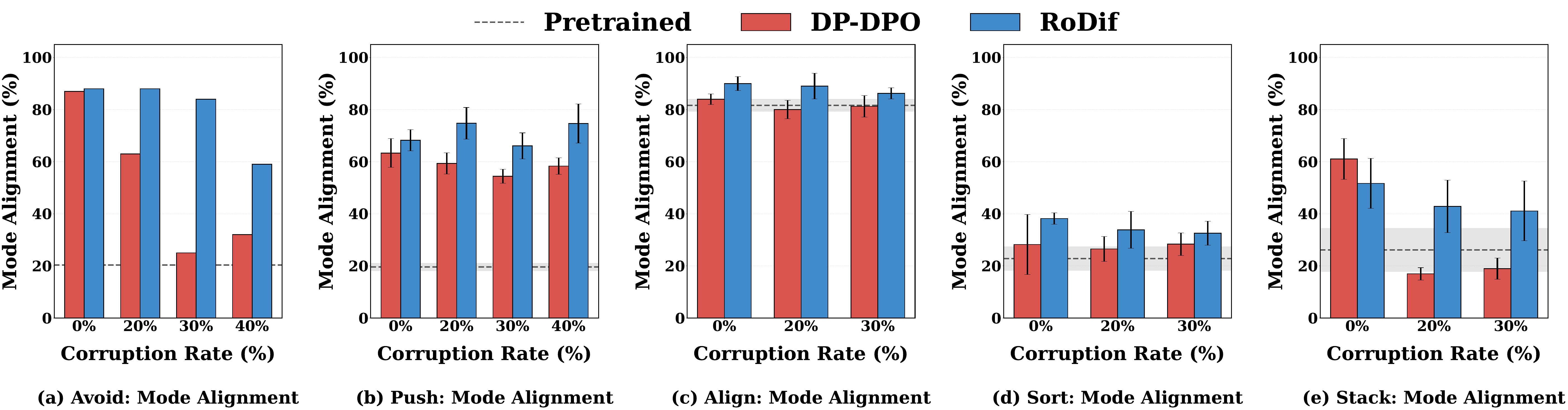}}
\caption{Mode alignment rate of the fine-tuned diffusion policies using RoDiF and DP-DPO under different corruption levels. DP-DPO utilizes the standard DPO loss (\ref{eq:dense_loss}) while keeping all other settings identical. For RoDiF, the hyperparameter settings were adapted based on the task characteristics. For the \textbf{Avoid} and \textbf{Push} tasks, we fixed $\alpha=1$ and $\beta=0.1$, while $\gamma$ was set according to the noise level: $\gamma=0$ for $0\%$ noise, $\gamma=0.3$ for $20\%$ noise, $\gamma=0.4$ for $30\%$ noise, and $\gamma=0.45$ for $40\%$ noise. For the \textbf{Sort}, \textbf{Align} and \textbf{Stack} tasks, we utilized fixed hyperparameters across all noise levels: $\alpha=1$, $\beta=0.005$, and $\gamma=0.45$.}
\label{fig:stitched-noise}
\vspace{-5pt}
\end{figure*}

\vspace{-5pt}
\subsection{Experimental Setup}
\label{subsec:setup}
\vspace{-5pt}
\paragraph{Benchmark Tasks}
We evaluate our method on five long-horizon manipulation tasks adapted from  \emph{D3IL benchmark} \cite{jia2024towards}, as shown in Fig. \ref{fig:task-setup}.
The details of the five benchmark tasks are listed in Table \ref{tab:task_details} and in Appendix \ref{Appendix:task_detail}.

\vspace{-8pt}
\paragraph{Diffusion Policy Architectures} We use three diffusion architectures (details in Appendix \ref{Appendix:policy_arch}): \textbf{(1) State-Based UNet} (\textit{Avoid, Pushing}): Standard Conditional UNet \cite{chi2025diffusion} predicting action chunks  from observation history.  \textbf{(2) Transformer ACT} (\textit{Sorting, Stacking}): Action Chunking Transformer \cite{zhao2023learning}  utilizing ResNet-18 features for visual feature extraction. \textbf{(3) Image-Based MLP} (\textit{Aligning}): Single-action diffusion  using the same ResNet-18 visual backbone as above. The training settings of all policies in each task is shown in Appendix \ref{Appendix:training}.
    


\vspace{-10pt}
\paragraph{Human Preference Label  Generation}
We use a pairing strategy to generate human preference labels. We collected a set of $N_{\text{win}}$ winner (preferred mode) trajectories and $N_{\text{lose}}$ loser trajectories for each task. We construct the preference dataset $\mathcal{D}_{\text{pair}}$ by forming all possible Cartesian products of winners and losers. This results in a total dataset size of $N_{\text{pair}} = N_{\text{win}} \times N_{\text{lose}}$ (See numbers for each tasks in Appendix \ref{Appendix:task_param}). To generate corrupted data labels, we randomly select a portion of pairs from $\mathcal{D}_{\text{pref}}$. For these pairs, the preference labels are permanently inverted. This mimics real-world scenarios where a portion of the data is mislabeled due to annotator error or ambiguity.


\vspace{-10pt}
\paragraph{Evaluation Metrics} We assess the fine-tuned diffusion policies by analyzing the behavioral modes manifested during rollout episodes for each task. We report results using \textbf{(1) Mode Alignment Rate:} The percentage of trajectories that match the preferred mode (last col. of Table \ref{tab:task_details}). \textbf{(2) Success Rate:} Percentage of episodes in which the robot successfully achieves the goal (second col. of Table \ref{tab:task_details}).

\vspace{-5pt}
\subsection{Results}
\label{subsec:performance}
\begin{table*}[!htbp]
    \centering
    \caption{{Overall performance (success rate and mode alignment rate) of the  diffusion policies fine-tuned by RoDiF} under different label corruption rates. $\gamma$, $\beta$ and $\alpha$ values follow the same setting in Fig. \ref{fig:stitched-noise}.}
    \label{tab:main_results}
    \resizebox{\linewidth}{!}{
    \begin{tabular}{l|c|cc|cc|cc|cc}
        \toprule
        & & \multicolumn{2}{c|}{\textbf{PreTrained}} & \multicolumn{2}{c}{\textbf{RoDif (0\% Corruption)}} & \multicolumn{2}{c}{\textbf{RoDif (20\% Corruption)}} & \multicolumn{2}{c}{\textbf{RoDif (30\% Corruption)}}\\
        \textbf{Task} & \textbf{Backbone} & \textbf{Success} & \textbf{Alignment} & \textbf{Success} & \textbf{Alignment} & \textbf{Success} & \textbf{Alignment} & \textbf{Success} & \textbf{Alignment}\\
        \midrule
        Avoid & UNet & 78.00 $\pm$ 0.00\% & 20.33 $\pm$ 0.00\% & \textbf{88.00 $\pm$ 0.00}\% & \textbf{88.00 $\pm$ 0.00}\% & 86.00 $\pm$ 0.00\% & 86.00 $\pm$ 0.00\% & 84.00 $\pm$ 0.00\% & 84.00 $\pm$ 0.00\% \\
        Pushing & UNet & 57.67 $\pm$ 1.25\% & 19.59 $\pm$ 1.45\% & 52.60 $\pm$ 2.06\% & 68.23 $\pm$ 4.06\% & \textbf{56.14 $\pm$ 2.47}\% & \textbf{74.77 $\pm$ 6.06}\% & 54.00 $\pm$ 1.41\% & 66.09 $\pm$ 4.99\% \\
        Sorting & Transformer & 82.00 $\pm$ 4.00\% & 22.82 $\pm$ 4.58\% & 83.85 $\pm$ 5.26\% & \textbf{38.17 $\pm$ 2.22}\% & 79.43 $\pm$ 2.97\% & 33.82 $\pm$ 7.07\% & \textbf{84.00 $\pm$ 6.00}\% & 32.55 $\pm$ 4.61\% \\
        Aligning & MLP & 5.00 $\pm$ 3.54\% & 81.67 $\pm$ 2.36\% & \textbf{12.00 $\pm$ 5.10}\% & \textbf{90.00 $\pm$ 2.67}\% & 8.33 $\pm$ 4.71\% & 89.00 $\pm$ 4.90\% & 8.33 $\pm$ 4.71\% & 86.25 $\pm$ 2.17\% \\
        Stacking & Transformer & 23.03 $\pm$ 2.27\% & 26.12 $\pm$ 8.34\% & \textbf{28.75 $\pm$ 12.44}\% & \textbf{51.66 $\pm$ 9.57}\% & 25.00 $\pm$ 10.80\% & 42.86 $\pm$ 10.10\% & 27.00 $\pm$ 6.00\% & 41.09 $\pm$ 11.47\% \\
        \bottomrule
    \end{tabular}
    }
\end{table*}
\paragraph{Overall Performance}

The overall results are reported in Table~\ref{tab:main_results} and Fig.~\ref{fig:stitched-noise}. In Table~\ref{tab:main_results}, we demonstrate the effectiveness of RoDiF for  fine-tuning  diffusion policy under different levels of preference label corruption. It shows that RoDif consistently steers the policy toward the user-specified modes, demonstrating robust performance even in challenging scenarios. Specifically, for four tasks \textbf{Avoid, Pushing, Sorting, and Stacking}, the target behavior mode corresponds to a \textbf{non-dominant mode} of the pretrained policy (baseline alignment $<27\%$). In these cases, the method must override the policy's strong natural tendency to execute alternative behaviors. Additional result and visualization is shown in Appendix \ref{Appendix:Additional}.

Table \ref{tab:main_results} also show that the success rates of the fine-tuned policies have slightly increased compared to the pretrained policies. This indicates that while steering the diffusion policy toward preferred mode,  RoDiF does not compromise the overall task performance, \textbf{maintaining task stability throughout the fine-tuning process}. Furthermore, Table~\ref{tab:main_results} demonstrates our method is effective across different diffusion backbones, validating its \textbf{architecture agnosticism}.

In Fig.~\ref{fig:stitched-noise}, we compare the performance of RoDiF against the baseline DP-DPO, which utilizes the standard DPO loss (\ref{eq:dense_loss}) while keeping all other settings identical. The results demonstrate that under 0\% label corruption, both methods fine-tune the policy with comparable performance. However, as the corruption rate increases, RoDiF demonstrates significantly greater robustness, maintaining high fine-tuning performance where the baseline DPO degrades.



\begin{figure}[!htbp]
\vspace{-5pt}
    \centering
    \includegraphics[width=1\linewidth]{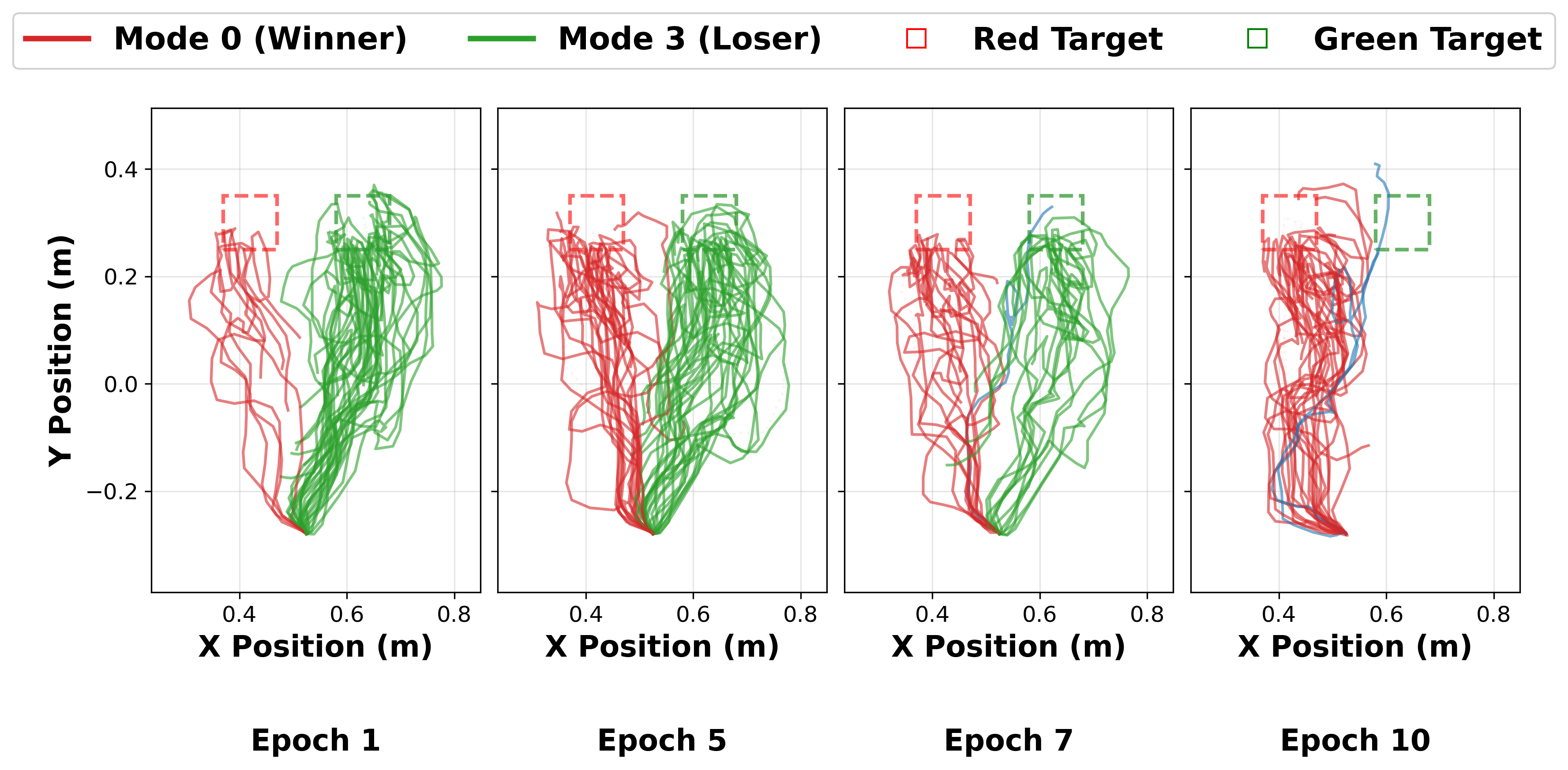}
    \vspace{-10pt}
    \caption{Visualizing behavior (policy rollouts) mode during fine-tuning in the pushing task.}
    \label{fig:progress}
    \vspace{-15pt}
\end{figure}


\paragraph{Qualitative Results.} Fig.~\ref{fig:progress} visualizes the evolution of the policy mode during fine-tuning for the \textit{Pushing} task, where the objective is to steer towards a non-dominant mode. At Epoch 1, the rollout trajectories are heavily clustered around the dominant mode (visualized in green), reflecting the bias of the pretrained backbone. However, as training progresses, the distribution shifts significantly, with the target non-dominant trajectories (visualized in red) becoming prevalent by the final epoch.


\begin{wrapfigure}[10]{r}{0.45\columnwidth}
    \centering
    \vspace{-20pt} 
    \includegraphics[width=\linewidth]{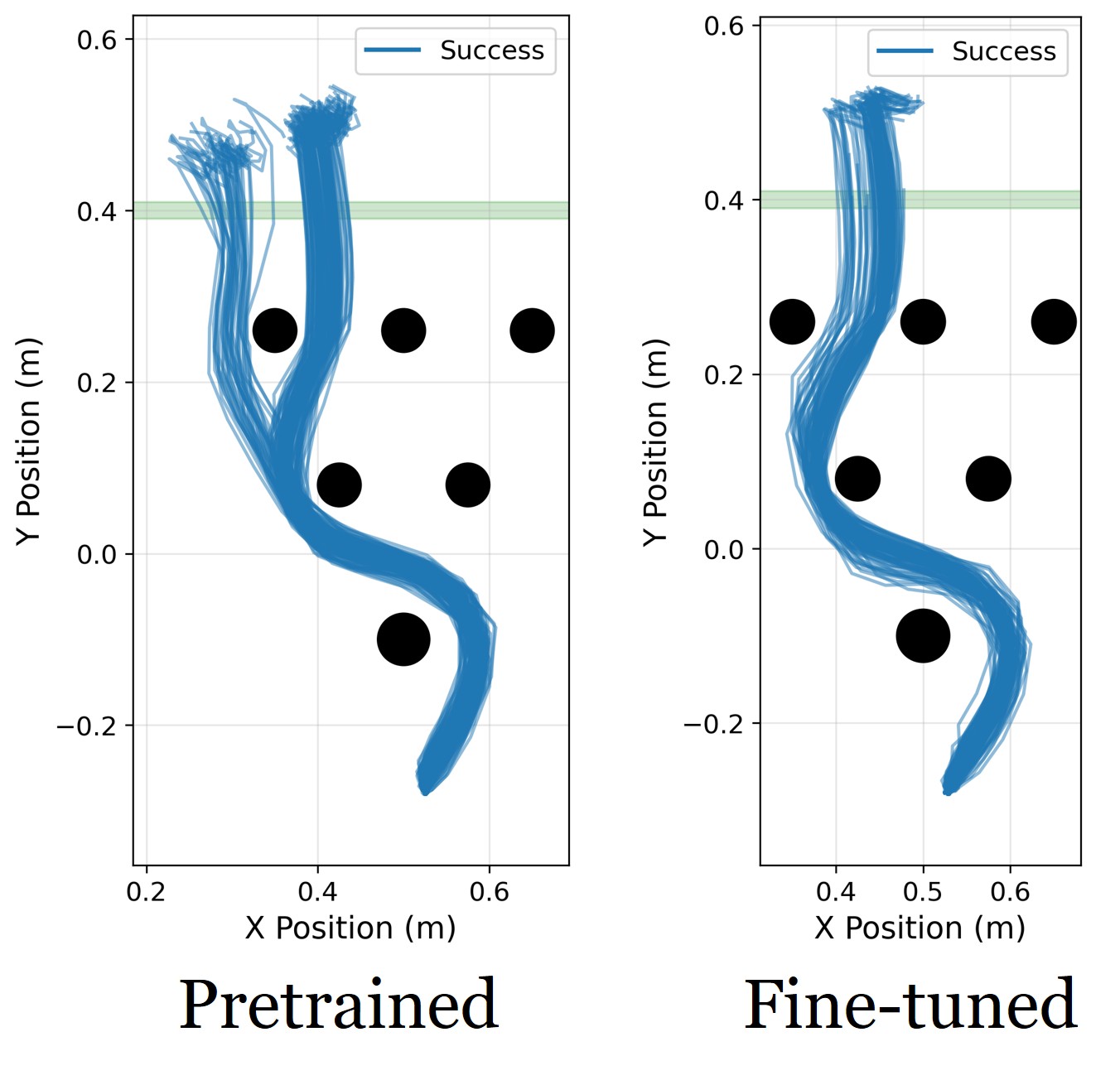}
    \vspace{-10pt}
    \caption{Mode visualization for avoid task}
    \label{fig:two_tasks_forced}
    \vspace{-10pt} 
\end{wrapfigure}
We also validate RoDif's capability to reinforce \textbf{dominant modes}. For the \textit{Avoid} task, when steered towards the already-preferred trajectory, the policy achieved a mode alignment of $84\%$ over 100 evaluation rollouts as shown in  Fig. \ref{fig:two_tasks_forced}, confirming it can sharpen existing behaviors as well.

Finally, we highlight the \textbf{sample efficiency} of our approach; for the \textit{Pushing} task, RoDiF use as few as 10 winner-loser pairs was sufficient to achieve a mode alignment rate of $76\%$ (see the visualization in Appendix \ref{Appendix:Additional}), significantly reducing the human preference annotation burden.


\begin{figure}[!htbp]
    \centering
    \begin{subfigure}[t]{0.49\linewidth}
        \centering
        \includegraphics[width=\linewidth]{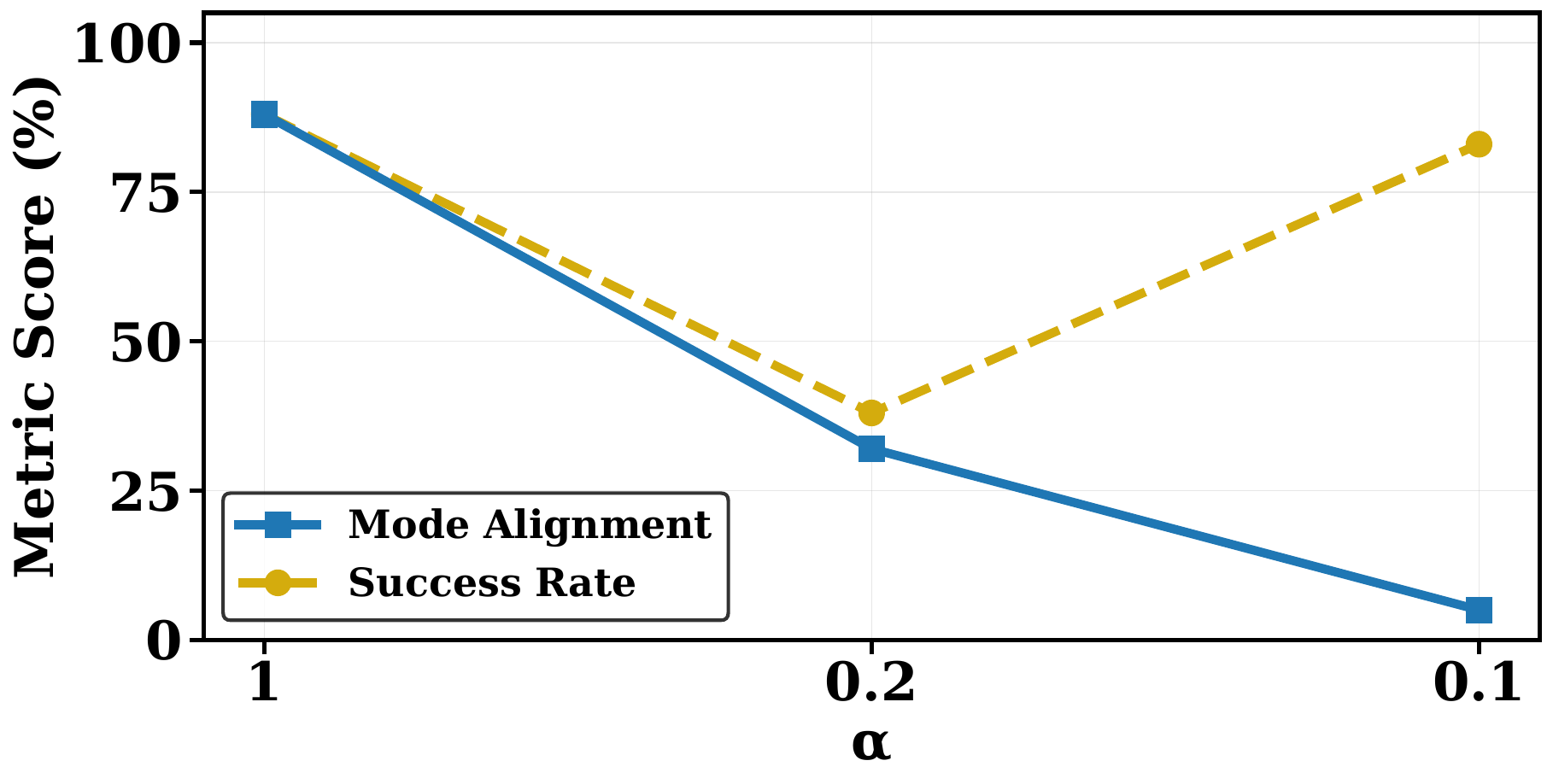}
        \caption{Effect of $\alpha$}
        \label{fig:ablation_alpha}
    \end{subfigure}
    \hfill 
    \begin{subfigure}[t]{0.49\linewidth}
        \centering
        \includegraphics[width=\linewidth]{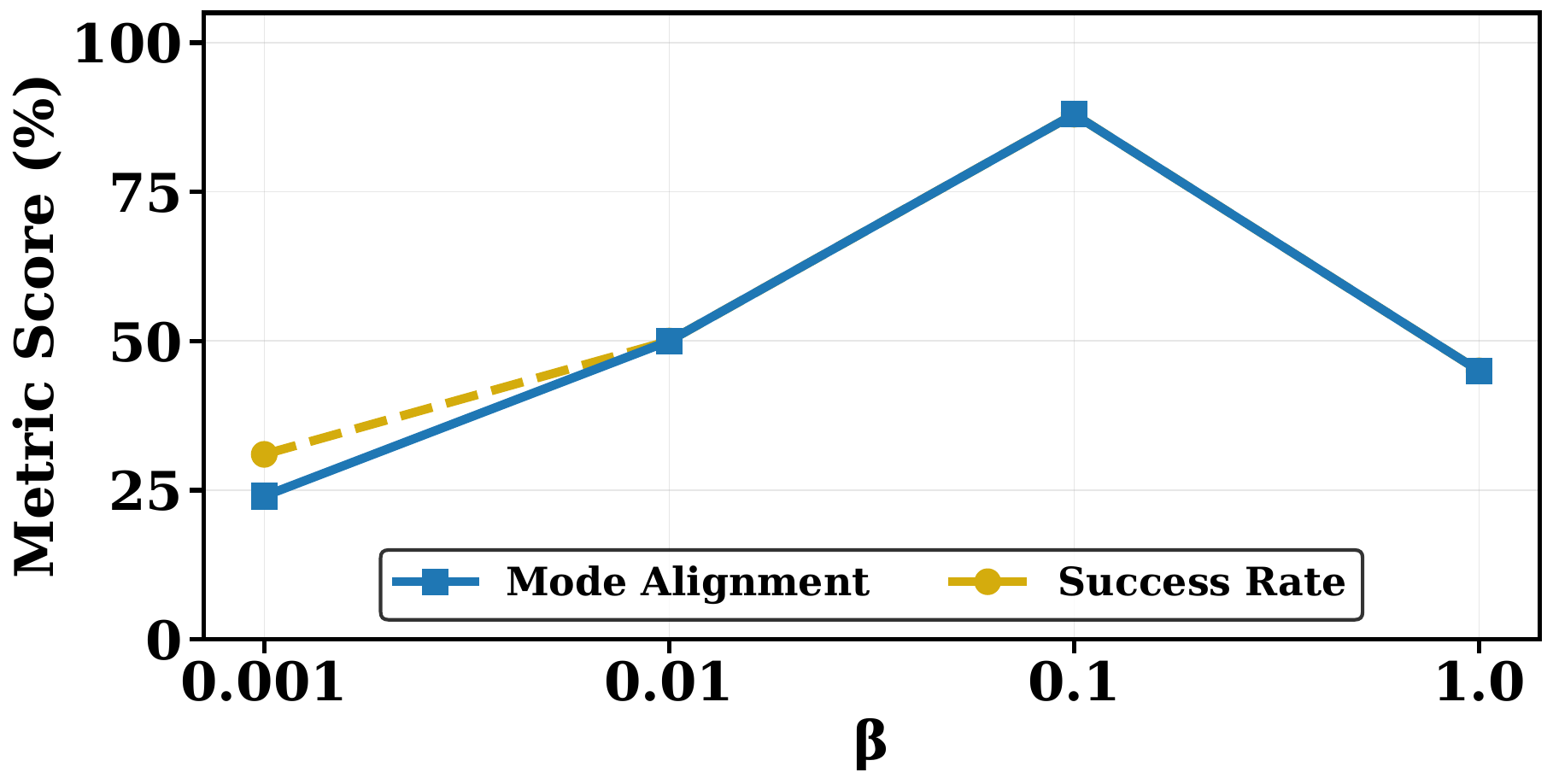}
        \caption{Effect of $\beta$}
        \label{fig:ablation_beta}
    \end{subfigure}
    \hfill
    \begin{subfigure}[t]{0.49\linewidth}
        \centering
        \includegraphics[width=\linewidth]{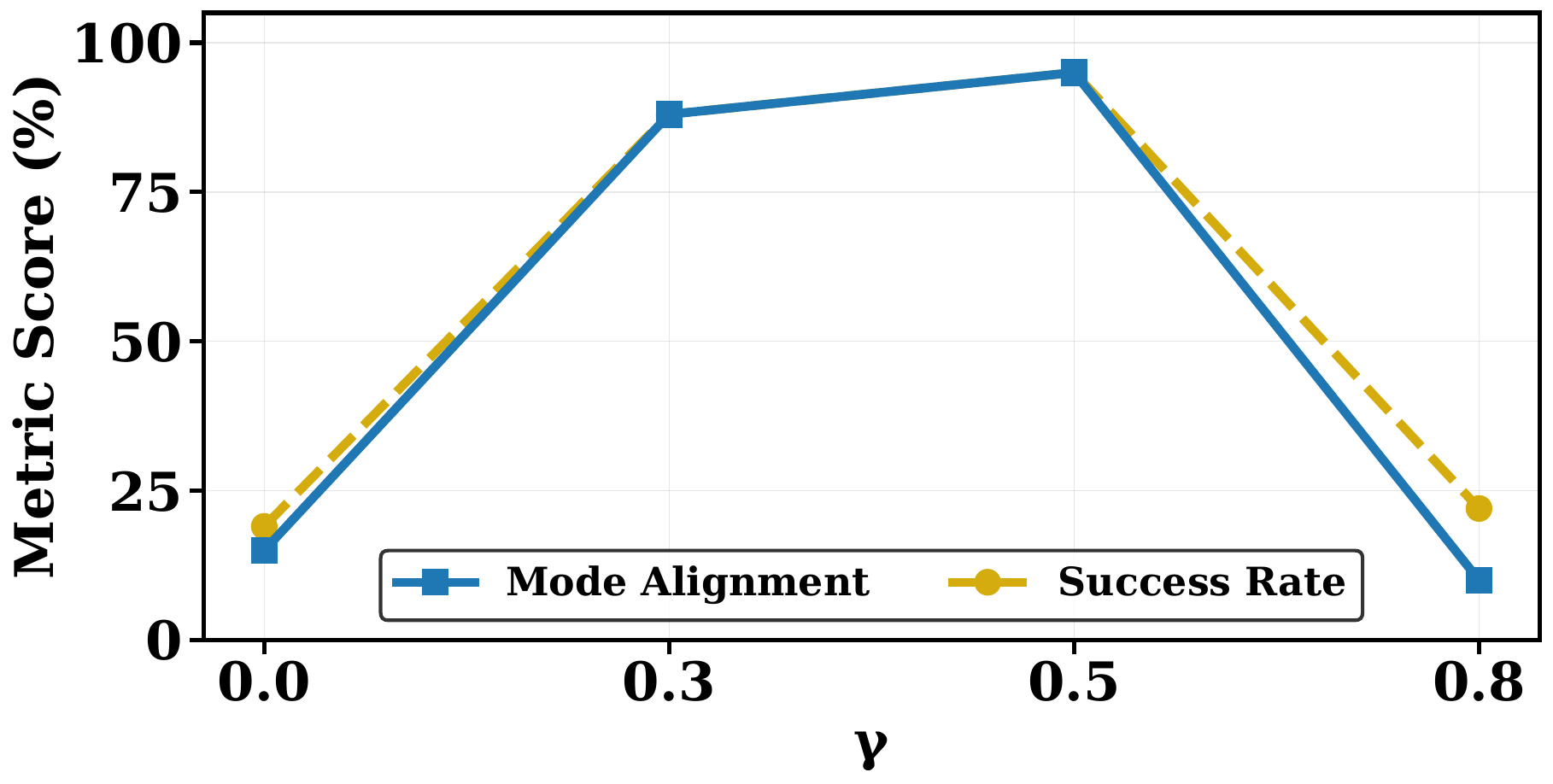}
        \caption{Effect of $\gamma$}
        \label{fig:ablation_gamma}
    \end{subfigure}
    \caption{Ablation studies on the sensitivity of the success rate and mode alignment to variations in $\beta$, $\alpha$, and $\gamma$.}
    \vspace{-10pt}
    \label{fig:all_ablations}
\end{figure}

\subsection{Baseline Comparisons} 
Table~\ref{tab:robustness_baselines} evaluates RoDif against baselines under varying noise. Margin-based \textbf{SimPO} \cite{meng2024simposimplepreferenceoptimization} collapses under 20\% noise (43\%$\to$10\% success), confirming that hard margins overfit to systematic corruption. \textbf{CPO} \cite{xu2024contrastivepreferenceoptimizationpushing} remains stable ($\sim$77\% success) but fails to steer effectively (low alignment), likely constrained by its behavior cloning term. \textbf{IPO} \cite{azar2023generaltheoreticalparadigmunderstand} proves inconsistent. In contrast, \textbf{RoDif} maintains robustness across all levels ($>$84\% success). At 20\% noise, RoDif achieves \textbf{86\% success}, outperforming the strongest baseline by +9 \% and demonstrating that it effectively isolates the steering signal from systematic label noise.

\begin{table}[h]
    \centering
    \caption{\textbf{Robustness Comparison.} Performance comparison of RoDif against baseline preference optimization methods }
    \label{tab:robustness_baselines}
    \resizebox{\linewidth}{!}{
    \begin{tabular}{l|ccc|ccc}
        \toprule
        & \multicolumn{3}{c|}{\textbf{Success Rate (\%)}} & \multicolumn{3}{c}{\textbf{Mode Alignment (\%)}} \\
        \textbf{Method} & \textbf{0\% } & \textbf{20\% corr.} & \textbf{30\% corr.} & \textbf{0\% corr.} & \textbf{20\% corr.} & \textbf{30\% corr.} \\
        \midrule
        \textbf{RoDif} & \textbf{88.00} & \textbf{86.00} & \textbf{84.00} & \textbf{88.00} & \textbf{86.00} & \textbf{84.00} \\
        CPO & 76.00 & 77.00 & 78.00 & 69.00 & 70.00 & 71.00 \\
        IPO & 65.00 & 49.00 & 69.00 & 63.00 & 47.00 & 58.00 \\
        SimPO & 43.00 & 10.00 & 27.00 & 43.00 & 9.00 & 21.00 \\
        \bottomrule
    \end{tabular}
    }
    \vspace{-10pt}
\end{table}


\subsection{Ablation Study}
\label{sec:ablation}

For the RoDiF method, we conducted a sensitivity analysis on the \textit{Avoid} task under a fixed label corruption rate of $20\%$. We examined the impact of three  hyperparameters: temperature $\alpha$,  KL weight $\beta$, and  conservativeness level  $\gamma$.

\textbf{Temperature ($\alpha$):} 
 Fig.~\ref{fig:ablation_alpha} shows that a softer (high temperature) ($\alpha{=}1$) significantly outperforms a sharp one ($\alpha{=}0.1$). Low temperature $\alpha$ makes the voting function (\ref{eq:single_vote}) near a binary step. This leads to \textit{gradient vanish} for learning signals. In contrast, $\alpha=1$ maintains necessary gradient flow, facilitating continuous optimization.

\textbf{KL-Constraint ($\beta$):} 
Fig.~\ref{fig:ablation_beta} reveals a trade-off governed by $\beta$. Large values ($\beta {=} 1$) induce gradient saturation, causing the policy to become trapped in the pretrained mode. Conversely, very small $\beta$ values lead to overly aggressive updates that degrade trajectory smoothness and leads to collision. Empirically, $\beta = 0.1$ provides the best balance for effective mode adaptation.

\textbf{Conservativeness level ($\gamma$):} 
 $\gamma$ is a conservative estimate of the actual corruption rate, which can be considered as the decision boundary for noise rejection. Given the actual corruption rate  $20\%$, Fig.~\ref{fig:ablation_gamma} shows that $\gamma$ should match or slightly surpass the actual corruption rate for RoDiF to work. This corresponds to the theory in Lemma \ref{lemma:robust}. At large levels like $\gamma$ = 0.8, the RoDiF becomes overly conservative, discarding valid but imperfect trajectories.

\vspace{-5pt}
\section{Conclusion}
\vspace{-5pt}
In this work, we propose RoDiF, a novel framework for robust direct fine-tuning of diffusion policies under corrupted human feedback. RoDiF introduces a unified MDP that jointly models the diffusion denoising process and environment dynamics, enabling the application of Direct Policy Optimization (DPO) to diffusion policies. To address noisy preference feedback, RoDiF reinterprets the DPO objective through a geometric hypothesis-cutting perspective and adopts a conservative cutting strategy, resulting in robust policy fine-tuning. Across diverse robotic tasks, RoDiF matches the performance of state-of-the-art preference-based fine-tuning methods under clean data and significantly outperforms them when a large fraction of preferences is corrupted.

\section*{Impact Statement}
This work advances the safety and reliability of preference-based robot learning by introducing a robust fine-tuning framework for diffusion policies under noisy human feedback. By explicitly mitigating the effects of erroneous or inconsistent preferences, the proposed approach reduces the risk of brittle or unstable policy updates, which is critical for deploying robotic systems in real-world, safety-sensitive environments. The ethical and societal implications of this work align with those broadly recognized in autonomous robotics and learning-from-human-feedback research. We view this contribution as a step toward safer and more dependable robotic learning systems, emphasizing robustness as a foundational requirement for responsible deployment.


\bibliography{example_paper}
\bibliographystyle{icml2026}

\newpage
\appendix
\onecolumn
\section{Proof of Lemmas}
\subsection{Proof of Lemma \ref{lemma:cuts}}\label{appendix:proof1}
\begin{proof}
    \eqref{eq:perfect_indicator} reaches the maximum value $N$ if and only if $\theta 
    \in \Theta \cap \bigcap_{i=1}^N \mathcal{C}_i \subseteq \mathcal{C}_j$, since all of the votes $V_{\mathcal{C}_i}(\theta), i=1,\dots,N$ is 1. Thus the solution set to \eqref{eq:perfect_indicator} is $\Theta \cap \bigcap_{i=1}^N \mathcal{C}_i \subseteq \mathcal{C}_j$.
    In noise-free case, combining the definition of $\theta^H$ and the cuts $\mathcal{C}_i$ in \eqref{eq:cut_i}. It is clear that $\theta^H \in \mathcal{C}_i,\ i=1,\dots,N$. Combining with $\theta^H \in \Theta$, we have \eqref{eq:perfect_indicator}.

    If the $j$-th preference pair is falsely labelled, the $j$-th cut $\mathcal{C}_j$ is defined as:
    \begin{equation}
        \mathcal{C}_j = \left\{\theta\in\Theta\,|\,\Delta Q_{\theta} \big(\mathbf{a}^{+}_{j,w}|\mathbf{s}_j^+,\mathbf{a}^{+}_{j,l}|\mathbf{s}_j^+\big) \geq 0 \right\}.
    \end{equation}
    By the definition of $\theta^H$, $\Delta Q_{\theta^H} \big(\mathbf{a}^{+}_{j,w}|\mathbf{s}_j^+,\mathbf{a}^{+}_{j,l}|\mathbf{s}_j^+\big) = \Delta Q\big(\mathbf{a}^{+}_{j,w}|\mathbf{s}_j^+,\mathbf{a}^{+}_{j,l}|\mathbf{s}_j^+\big) < 0$. We have $\theta^H \notin \mathcal{C}_j$ and therefore $\theta^H \notin \Theta \cap \bigcap_{i=1}^N \mathcal{C}_i$. In other words, $\theta^H$ cannot be found in the above solution set. 
\end{proof}

\subsection{Proof of Lemma \ref{lemma:robust}} \label{appendix:proof2}
\begin{proof}
    With the definition of $\gamma$, there are at most $\lceil \gamma N \rceil$ incorrect labels and at least $\lfloor(1{-}\gamma) N\rfloor$ correct labels in the  dataset.  As a result, for $\theta^H$, there will be at least $\lfloor(1-\gamma) N\rfloor$ votes, i.e., $\sum_{i=1}^{N}V_{\mathcal{C}_i}(\theta^H) \geq \lfloor(1{-}\gamma)N\rfloor$ and $\mathbf{H}\left(\sum_{i=1}^{N}V_{\mathcal{C}_i}(\theta^H) {-} \lfloor(1{-}\gamma)N\rfloor {+} 0.5 \right)=1$, which means $\theta^H$ falls in the solution set in \eqref{eq:robust_dpo_loss}.
\end{proof}

\section{Experiment Setup Details}

\subsection{Task Details}
\label{Appendix:task_detail}
\textbf{1) Avoid (State-Based)} \\
The robot must navigate to a green finish line while avoiding six static obstacles.
\begin{itemize}
    \item \textbf{Observation:} 4-dimensional state (end-effector position and desired goal).
    \item \textbf{Modes \& Preference:} The pretrained policy exhibits multimodal behavior, passing obstacles from either the \textit{left} or \textit{right}. 
\end{itemize}

\textbf{2) Pushing (State-Based)} \\
The robot must push two blocks (red and green) into their corresponding target zones.
\begin{itemize}
    \item \textbf{Observation:} 10-dimensional state (end-effector pose, block positions, block orientations).
    \item \textbf{Modes \& Preference:} The task allows for diverse sorting order: \textbf{Green Block to Green Position or Red Block to Red Position} .
\end{itemize}

\textbf{3) Sorting (State-Based)} \\
The robot must sort one red block and one blue block into their matching target boxes.
\begin{itemize}
    \item \textbf{Observation:} $(4+3X)$-dimensional state where $X=2$ blocks.
    \item \textbf{Preference:} We enforce a temporal preference, defining the sorting sequence: \textbf{Red block then Blue block} or \textbf{Blue Block then Red Block}.
\end{itemize}

\textbf{4) Reorientation (Image-Based)} \\
The robot must push a hollow box to a predefined target position and orientation.
\begin{itemize}
    \item \textbf{Observation:}  The policy conditions on two $96 \times 96$ RGB images (Front View + In-Hand View) along with the robot's proprioceptive state.
    \item \textbf{Preference:} The box can be pushed from the \textit{inside} or \textit{outside}.
\end{itemize}

\textbf{5) Stacking (Image-Based)} \\
The robot must stack three colored blocks (Red, Green, Blue) into a tower.
\begin{itemize}
    \item \textbf{Observation:}  Conditions on two $96 \times 96$ RGB images (Front View + In-Hand View) and proprioceptive state.
    \item \textbf{Preference:} With six possible stacking permutations, we select the order \textbf{Red $\rightarrow$ Blue $\rightarrow$ Green} and \textbf{Green $\rightarrow$ Blue $\rightarrow$ Red} for preference labelling.
\end{itemize}

\subsection{Policy Architectures}
\label{Appendix:policy_arch}
\begin{itemize}[leftmargin=*, nosep]
    \item \emph{State-Based UNet:} For \textit{Avoid} and \textit{Pushing}, we employ the standard Conditional UNet architecture \cite{chi2025diffusion}. The model takes a history of observations ($T_{\text{obs}}=2$) and predicts  future action chunk of length ($T_{\text{pred}}=16$).
    
    \item \emph{Transformer ACT:} For \textit{Sorting}  and \textit{Stacking}, we use the Action Chunking Transformer (ACT)  architecture \cite{zhao2023learning}. This model processes the observation history ($T_{\text{obs}}=5$) and generate action chunk ($T_{\text{pred}}=16$). For image observation, visual features  from two camera views  are extracted with ResNet-18 encoders \cite{he2016deep} and are concatenated with  robot's proprioceptive state.

    \item \emph{Image-Based MLP:} For \textit{Aligning}, we employ a Single-Action diffusion architecture. Visual features from two camera views  are extracted with ResNet-18 encoders and concatenated with the robot's proprioceptive state. The model takes as input the current observation  ($T_{\text{obs}}=1$) to predict a single next action ($T_{\text{pred}}=1$).
\end{itemize}
\subsection{Training settings}
\label{Appendix:training}
For pretraining of policies, we used the same parameters, given in the D3IL benchmark repository.

\begin{table}[h]
    \centering
    \caption{\textbf{Training Hyperparameters } Summary of key parameters used across different tasks.}
    \label{tab:hyperparameters}
    \resizebox{0.95\linewidth}{!}{
    \begin{tabular}{l|c|c|c|c|c}
        \toprule
        \textbf{Parameter} & \textbf{Avoid} & \textbf{Pushing} & \textbf{Sorting} & \textbf{Aligning} & \textbf{Stacking} \\
        \midrule
        
        Denoising Steps ($K$) & 20 & 20 & 16 & 4 & 16 \\
        Optimizer & Adam & Adam & AdamW & AdamW & AdamW \\
        Learning Rate & $3 \times 10^{-5}$ & $3 \times 10^{-5}$ & $1 \times 10^{-5}$ & $1 \times 10^{-5}$ & $1 \times 10^{-5}$ \\
        Batch Size & 64 & 64 & 64 & 32 & 32 \\
        KL Coefficient ($\beta$) & 0.1 & 0.1 & 0.005 & 0.005 & 0.005 \\
        Robustness Temp ($\alpha$) & 1.0 & 1.0 & 1.0 & 1.0 & 1.0 \\

        \bottomrule
    \end{tabular}
    }
\end{table}

\subsection{Task Specific Parameters}
\label{Appendix:task_param}
The parameters $N_{\text{win}},N_{\text{lose}}$ and number of trials in evaluation is shown in Table \ref{tbl:task_specific}.

\begin{table}[!htbp]
    \centering
    \caption{\textbf{Dataset and Evaluation Statistics.} Number of winner and loser trajectories collected for preference finetuning, along with the number of evaluation rollouts used to compute success and alignment rates.}
    \label{tbl:task_specific}
    \label{tab:data_stats}
    \resizebox{0.5\linewidth}{!}{
    \begin{tabular}{l|c|c|c}
        \toprule
        \textbf{Task} & \textbf{\# Winners} $N_{\text{win}}$ & \textbf{\# Losers} $N_{\text{lose}}$ & \textbf{\# Eval. Rollouts} \\
        \midrule
        Avoid & 20 & 20 & 100 \\
        Pushing & 30 & 30 & 100 \\
        Sorting & 50 & 50 & 50 \\
        Aligning & 40 & 40 & 20 \\
        Stacking & 31 & 65 & 20 \\
        \bottomrule
    \end{tabular}
    }
\end{table}

\section{Additional Results}
\label{Appendix:Additional}
We conduct additional experiments by extending the noise level up to 0.95 on the pushing task. Representative evaluation trajectories are shown in Fig.~\ref{fig:stitched-noise_noise}. The results indicate that RoDiF maintains strong performance even when the corruption rate approaches 0.45. When the noise level exceeds 0.5, the aligned mode flips, which is expected since the dominant preference label is reversed.

We additionally visualize the robot end-effector trajectories for the avoidance task, comparing the performance of DP-DPO and RoDiF under preference noise. The results, shown in Fig.~\ref{fig:noise_levels_DPO}, clearly demonstrate that RoDiF better preserves the preferred behavior mode under high noise levels.

The visualization of the policy rollouts with 10 preference pairs is shown in Fig. \ref{fig:robust10}.

\begin{figure*}[!htbp]
\begin{center}
\centerline{\includegraphics[width=\textwidth]{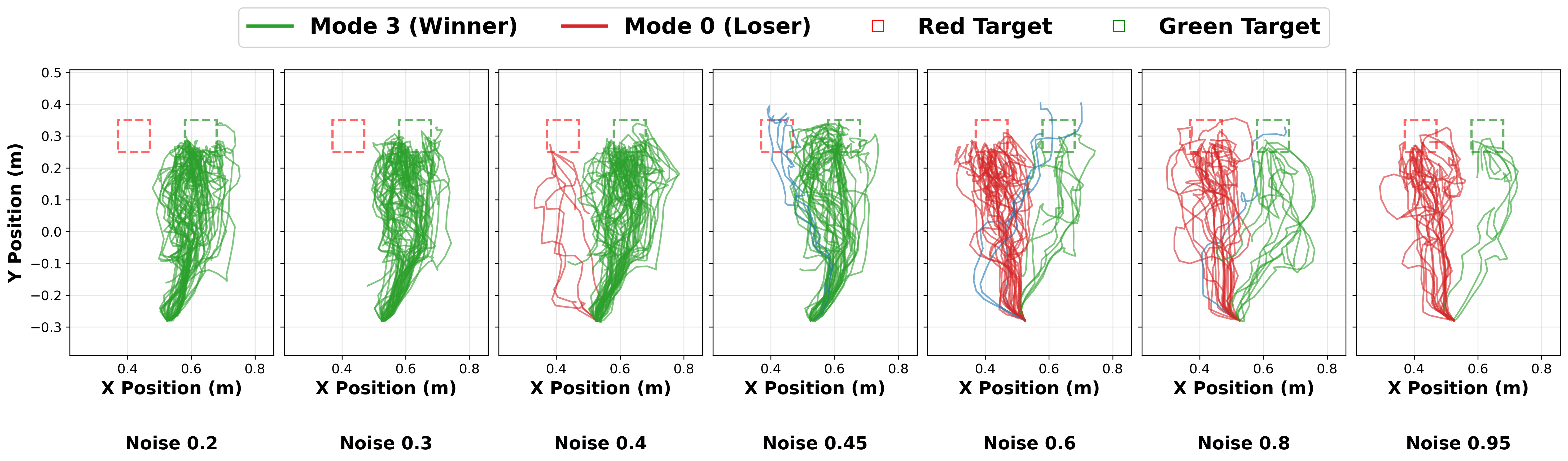}}
\caption{Mode visualization of the fine-tuned policies under different corruption rates. The regularization parameter $\gamma$ was set as follows: $\gamma=0.25$ for $20\%$ noise, $\gamma=0.35$ for $30\%$ noise, $\gamma=0.45$ for $40\%$ noise, $\gamma=0.5$ for $45\%$ noise, $\gamma=0.65$ for $60\%$ noise, $\gamma=0.85$ for $80\%$ noise, and $\gamma=0.95$ for $95\%$ noise.}
\label{fig:stitched-noise_noise}
\end{center}
\end{figure*}


\begin{figure*}[!htbp]
    \centering
    \begin{subfigure}[t]{0.25\textwidth}
        \centering
        \includegraphics[height=5cm, keepaspectratio]{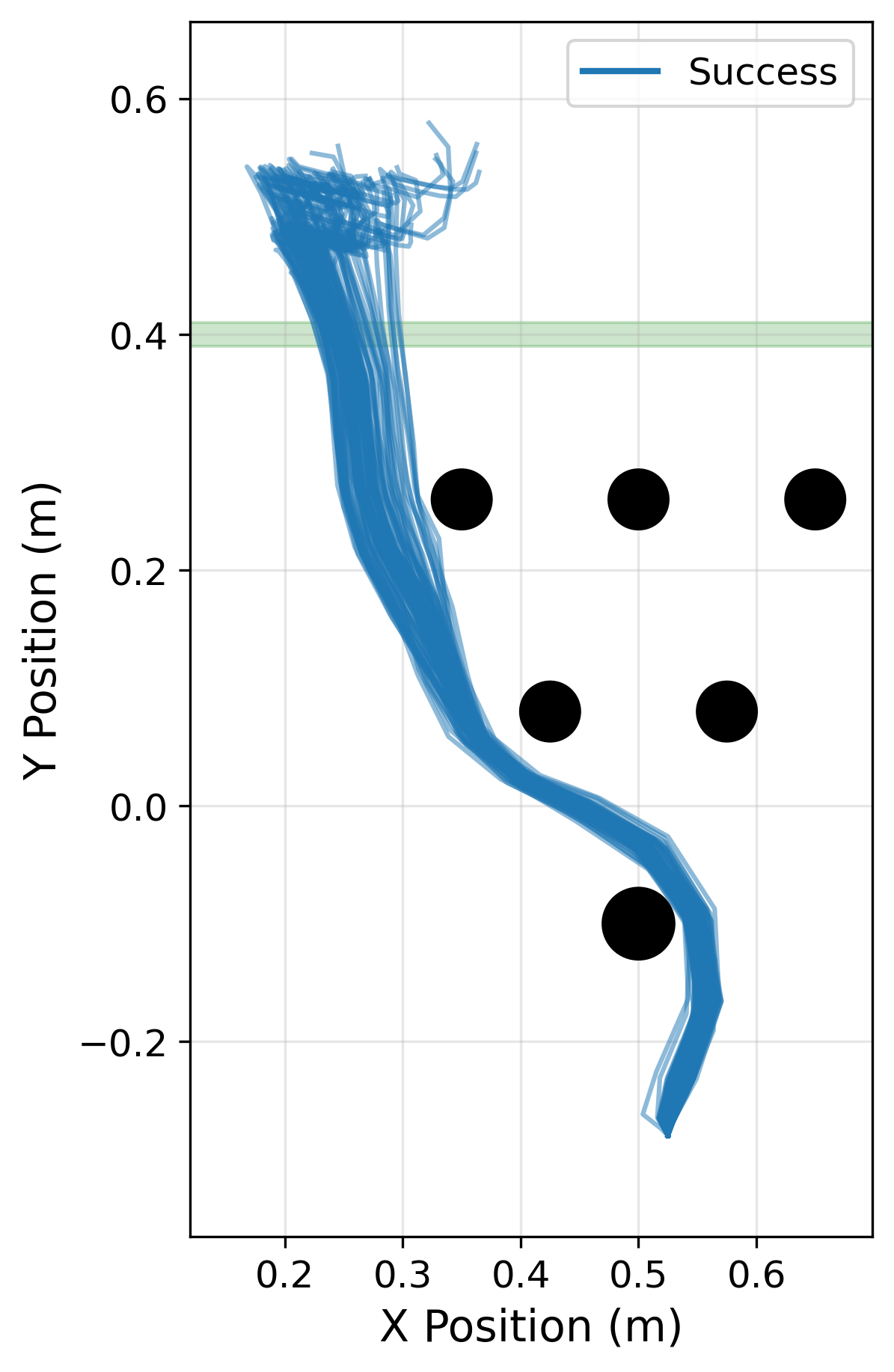}
        \caption{RoDif 0\%}
        
    \end{subfigure}\hfill
    \begin{subfigure}[t]{0.25\textwidth}
        \centering
        \includegraphics[height=5cm, keepaspectratio]{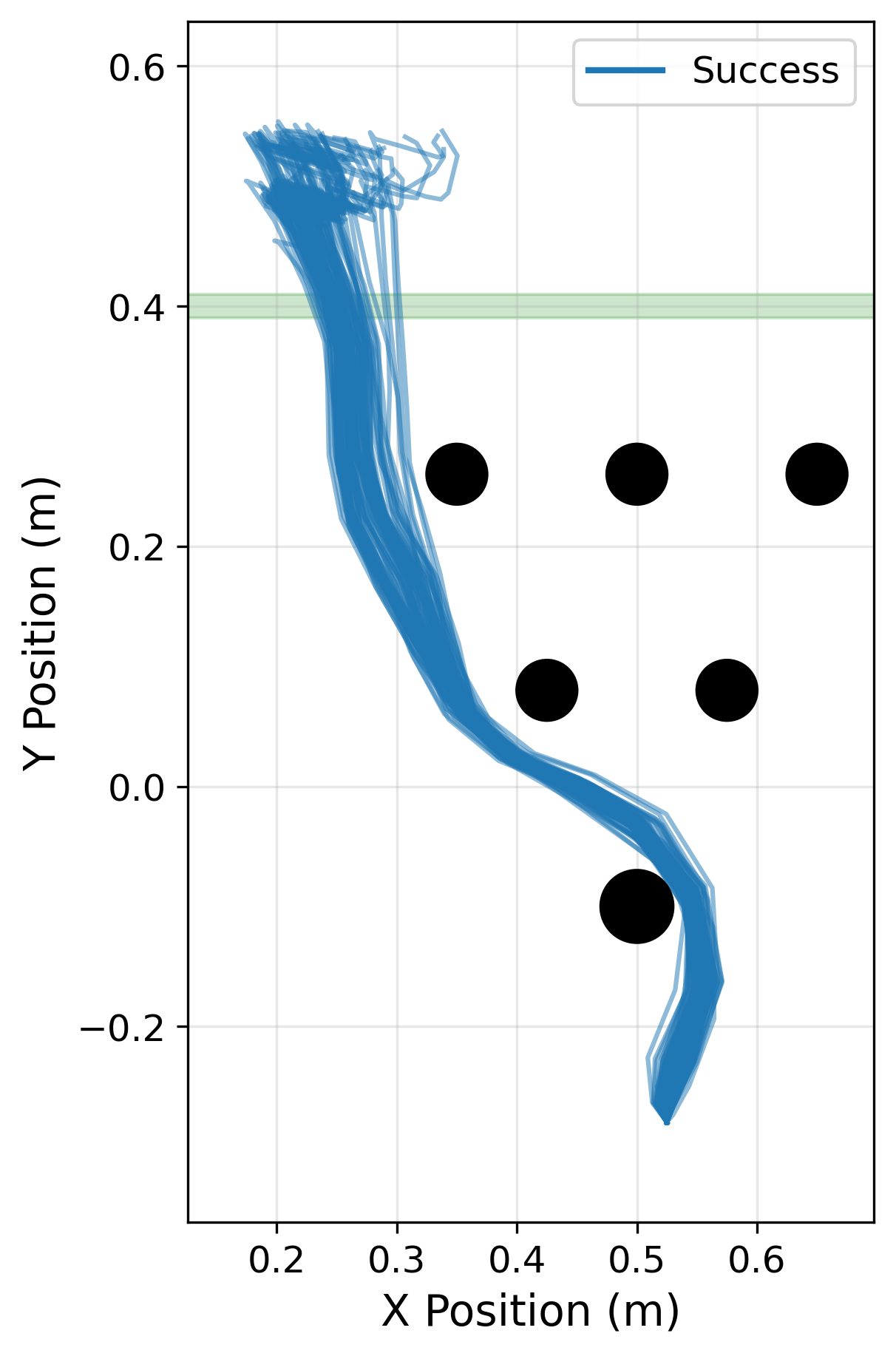}
        \caption{RoDif 20\%}
        
    \end{subfigure}\hfill
    \begin{subfigure}[t]{0.25\textwidth}
        \centering
        \includegraphics[height=5cm, keepaspectratio]{Final_fig/robys0.2.png}
        \caption{RoDif 30\%}
       
    \end{subfigure}\hfill
    \begin{subfigure}[t]{0.25\textwidth}
        \centering
        \includegraphics[height=5cm, keepaspectratio]{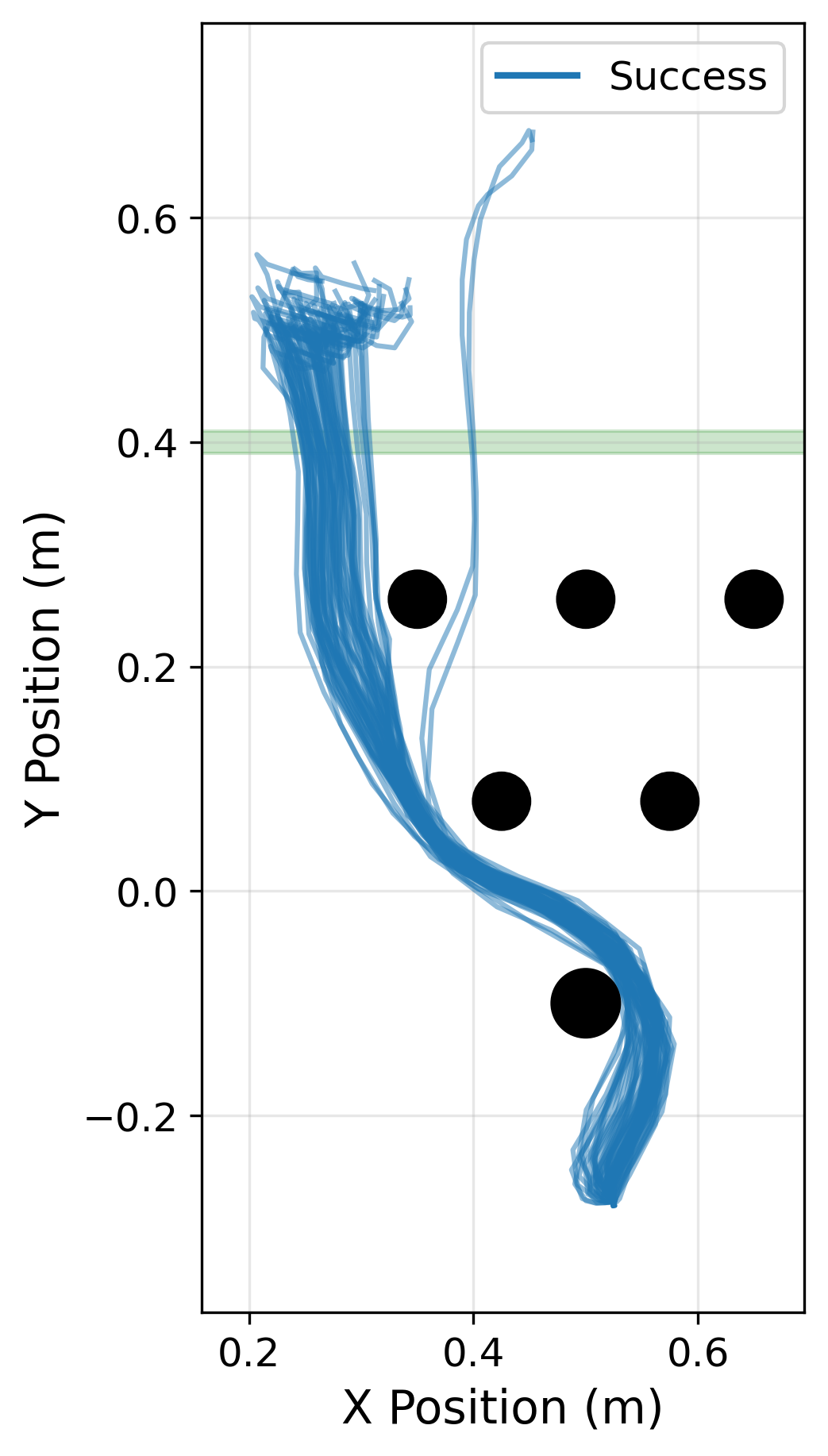}
        \caption{RoDif 40\%}
        
    \end{subfigure}

    \vspace{0.2cm} 

    \begin{subfigure}[t]{0.25\textwidth}
        \centering
        \includegraphics[height=5cm, keepaspectratio]{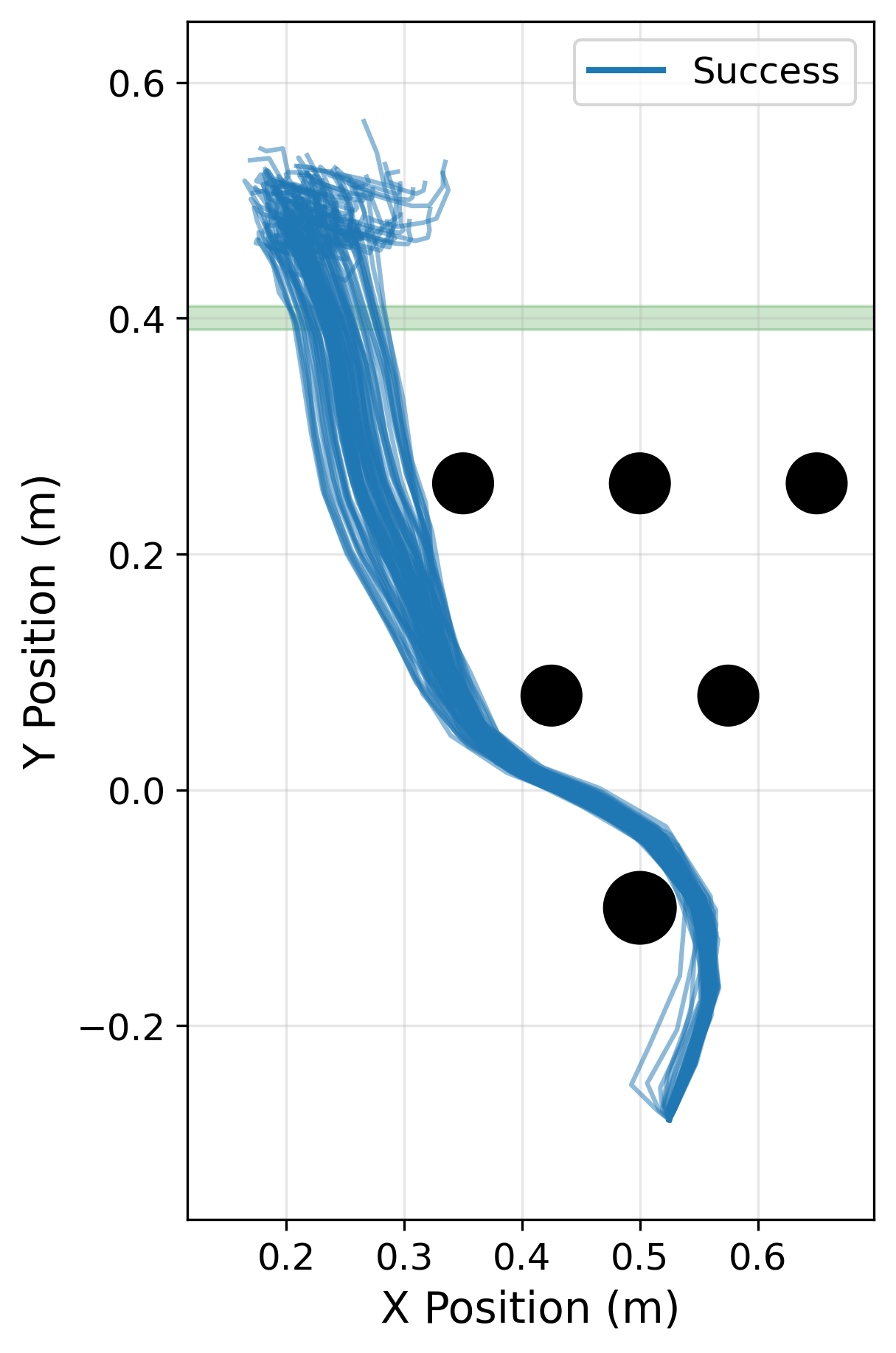}
        \caption{DP-DPO 0\%}
        
    \end{subfigure}\hfill
    \begin{subfigure}[t]{0.25\textwidth}
        \centering
        \includegraphics[height=5cm, keepaspectratio]{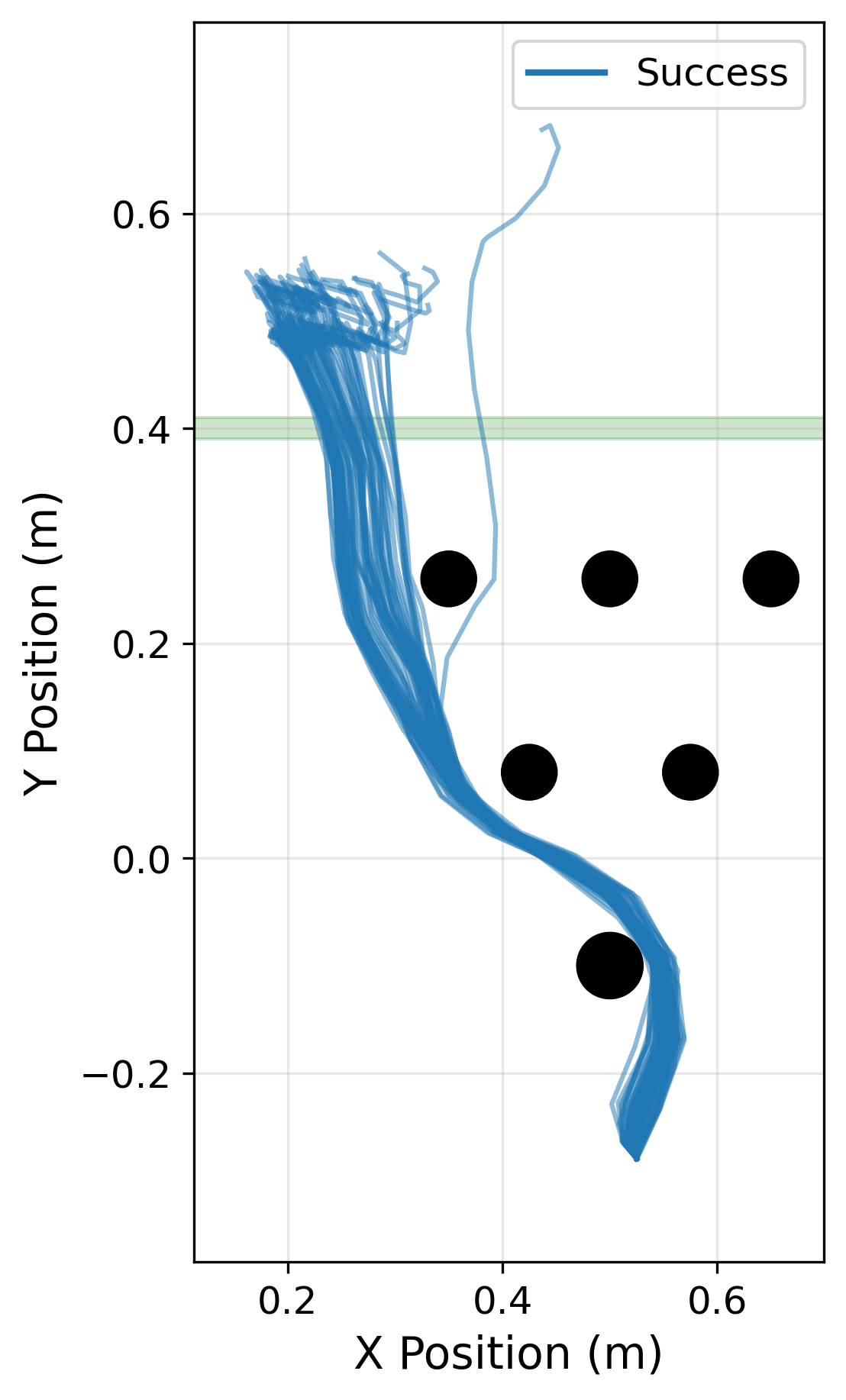} 
        \caption{DP-DPO 20\%}
        
    \end{subfigure}\hfill
    \begin{subfigure}[t]{0.25\textwidth}
        \centering
        \includegraphics[height=5cm, keepaspectratio]{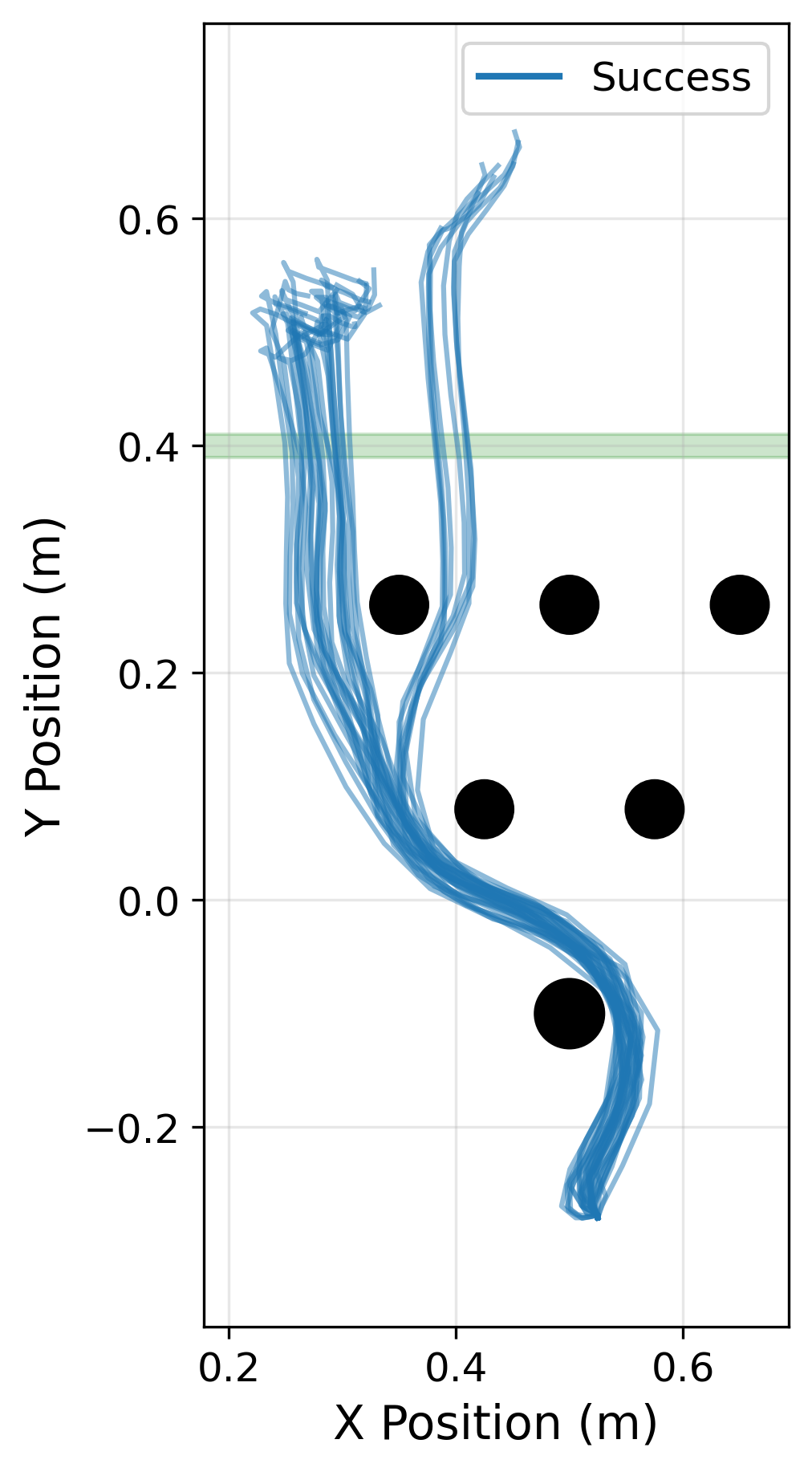}
        \caption{DP-DPO 30\%}
       
    \end{subfigure}\hfill
    \begin{subfigure}[t]{0.25\textwidth}
        \centering
        \includegraphics[height=5cm, keepaspectratio]{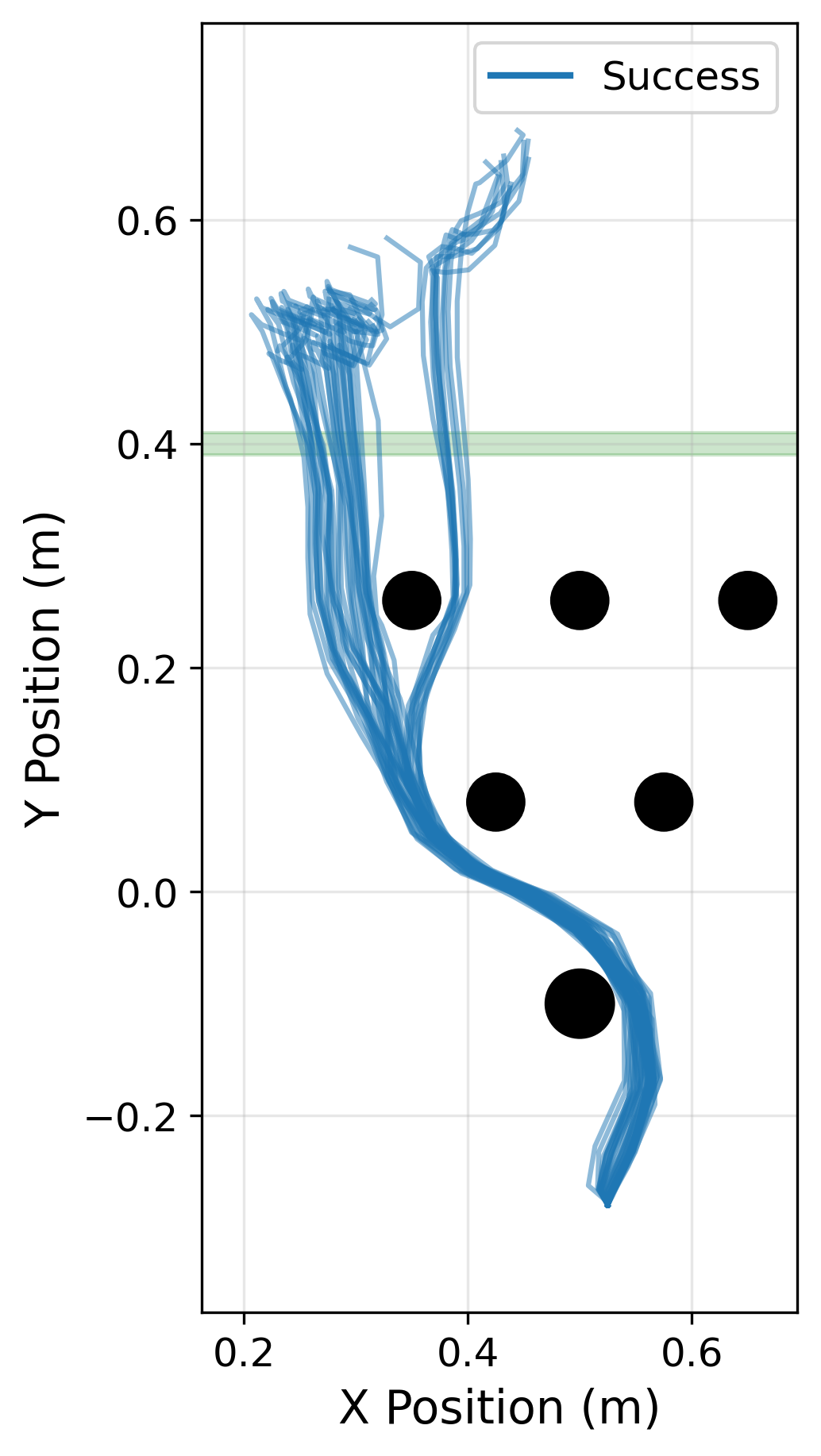}
        \caption{DP-DPO 40\%}

    \end{subfigure}

    \caption{Mode visualization of the fine-tuned policy using RoDif and DP-DPO (\ref{eq:dense_loss}) across different corruption rate for Avoid Task. Avoid from left mode was selected as winner mode for fine-tuning. DP-DPO consists of loser modes even after fine-tuning the pretrained policy at different noise levels.}
    \label{fig:noise_levels_DPO}
\end{figure*}



\begin{figure}[!htbp]
    \centering
    \includegraphics[width=0.5\linewidth]{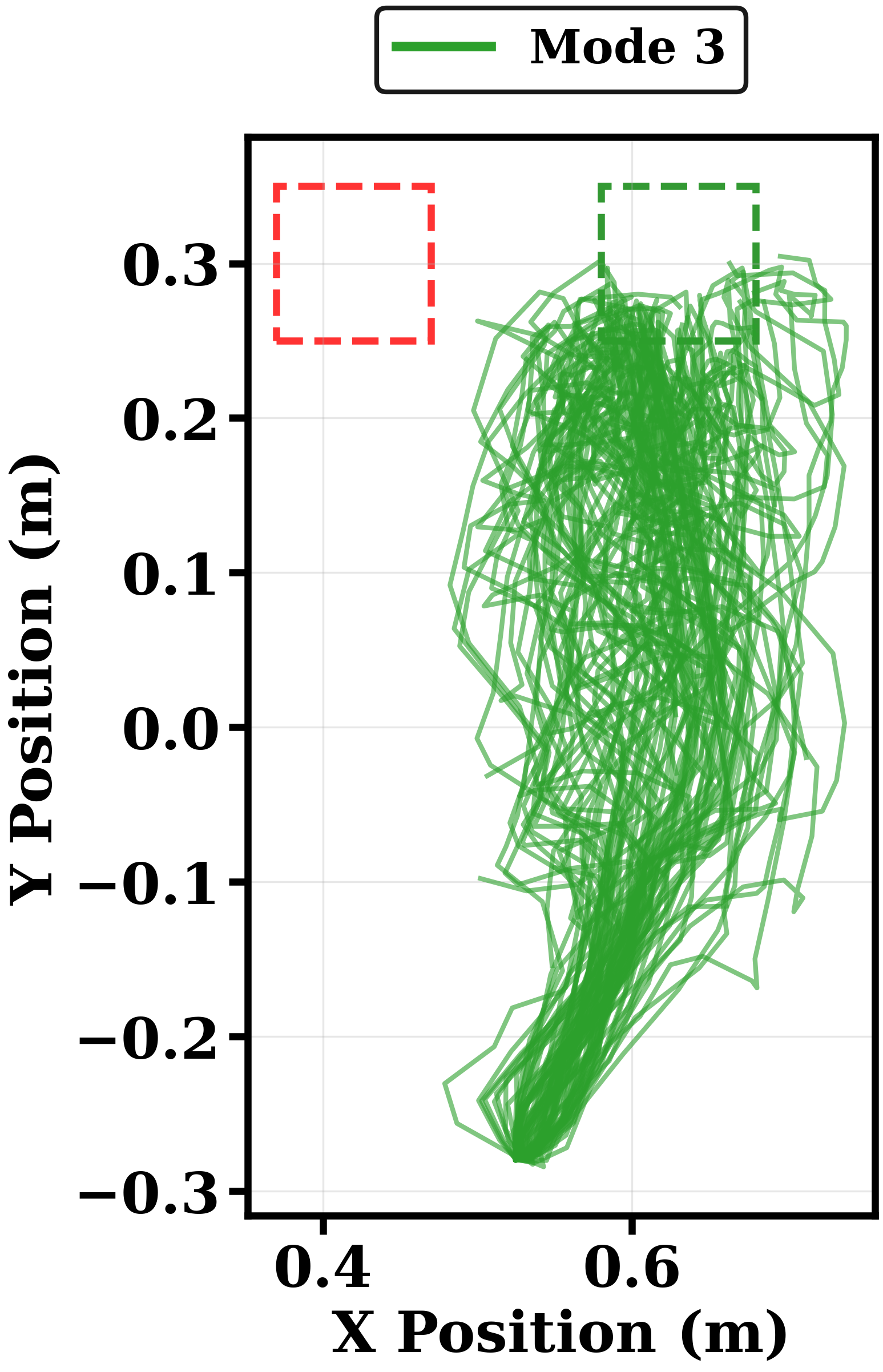}
    \caption{10 Winner and 10 Loser fine-tuning. Mode 3 i.e. pushing green block to green position was seleted as winner mode for fine-tuning.}
    \label{fig:robust10}
\end{figure}

\end{document}